\documentclass{article}

\usepackage[main,final]{neurips_2026}

\usepackage[utf8]{inputenc} 
\usepackage[T1]{fontenc}    
\usepackage{courier}        
\usepackage{hyperref}       
\usepackage{url}            
\usepackage{booktabs}       
\usepackage{amsfonts}       
\usepackage{nicefrac}       
\usepackage{microtype}      
\usepackage{xcolor}         

\usepackage{amsmath}
\usepackage{amssymb}
\usepackage{mathtools}
\usepackage{amsthm}
\usepackage{enumitem}
\usepackage{graphicx}    
\usepackage{subfig}      
\usepackage[table]{xcolor}
\usepackage{algorithm}
\usepackage{algorithmic}

\usepackage[capitalize,noabbrev]{cleveref}
\usepackage{tcolorbox}
\usepackage{fontawesome5}

\hypersetup{
  hidelinks,
  pdftitle={Heterogeneous Agent Collaborative Reinforcement Learning},
  pdfauthor={Zhixia Zhang, Zixuan Huang, Gongxun Li, Huaiyang Wang, Chengyi Yuan, Xin Xia, Deqing Wang, Fuzhen Zhuang, Shuai Ma, Ning Ding, Yaodong Yang, Jianxin Li, Yikun Ban}
}

\theoremstyle{plain}
\newtheorem{theorem}{Theorem}[section]

\newtheorem{lemma}[theorem]{Lemma}
\newtheorem{corollary}[theorem]{Corollary}
\theoremstyle{definition}
\newtheorem{definition}[theorem]{Definition}
\newtheorem{assumption}[theorem]{Assumption}
\theoremstyle{remark}
\newtheorem{remark}[theorem]{Remark}

\title{Heterogeneous Agent Collaborative Reinforcement Learning}

\author{%
  \textbf{Zhixia Zhang}\textsuperscript{1}\thanks{Equal contribution. Feel free to contact \texttt{zhixiaz@buaa.edu.cn} for details of the paper.} \quad
  \textbf{Zixuan Huang}\textsuperscript{1,2}\footnotemark[1] \quad
  \textbf{Gongxun Li}\textsuperscript{1} \quad
  \textbf{Huaiyang Wang}\textsuperscript{1} \quad
  \textbf{Chengyi Yuan}\textsuperscript{5} \quad \\
  \textbf{Xin Xia}\textsuperscript{2} \quad
  \textbf{Deqing Wang}\textsuperscript{1} \quad
  \textbf{Fuzhen Zhuang}\textsuperscript{1}\thanks{School of Artificial Intelligence, Beihang University.} \quad
  \textbf{Shuai Ma}\textsuperscript{1} \quad
  \textbf{Ning Ding}\textsuperscript{3} \\
  \textbf{Yaodong Yang}\textsuperscript{4} \quad
  \textbf{Jianxin Li}\textsuperscript{1} \quad
  \textbf{Yikun Ban}\textsuperscript{1}\thanks{Corresponding author: \texttt{yikunb@buaa.edu.cn}. State Key Laboratory of Complex \& Critical Software Environment, Beihang University.} \\
  \\
  \fontfamily{ptm}\selectfont
  \textsuperscript{1}\textbf{\textit{Beihang University}} \quad
  \textsuperscript{2}\textbf{\textit{Bytedance China}} \\
  \textsuperscript{3}\textbf{\textit{Tsinghua University}} \quad
  \textsuperscript{4}\textbf{\textit{Peking University}} \quad
  \textsuperscript{5}\textbf{\textit{Apple}} \\
  \\
  \href{https://github.com/zzx-peter/HACRL-code}{\faGithub\ code: \texttt{https://github.com/zzx-peter/HACRL-code}}
}

\begin{document}

\maketitle

\begin{abstract}
We introduce \textbf{H}eterogeneous \textbf{A}gent \textbf{C}ollaborative \textbf{R}einforcement \textbf{L}earning (\textbf{HACRL}), a new Reinforcement Learning from Verifiable Reward (RLVR) problem that addresses the inefficiencies of isolated multi-agent on-policy optimization.
HACRL targets model portfolios that serve a shared task distribution under different latency, hardware, and cost constraints: heterogeneous agents share verified rollouts during training to mutually improve, while operating independently at inference time.
Building on this problem, we propose \textbf{HACPO}, a collaborative RL algorithm that enables principled rollout sharing to maximize sample utilization and cross-agent knowledge transfer.
To mitigate capability discrepancies and policy distribution shifts, HACPO introduces four tailored mechanisms with theoretical guarantees for the oracle baseline and finite-batch capability ratio estimation.
Extensive experiments across diverse heterogeneous model combinations and reasoning benchmarks show that HACPO consistently improves all participating agents, outperforming GSPO with double rollouts by an average of 3.6\% while using only half the rollout cost. A three-agent experiment spanning three model families further demonstrates the effectiveness of heterogeneous collaboration.
\end{abstract}
\begin{figure} [t]
    \vskip -0.1in
    \centering
    \includegraphics[width=0.95\textwidth]{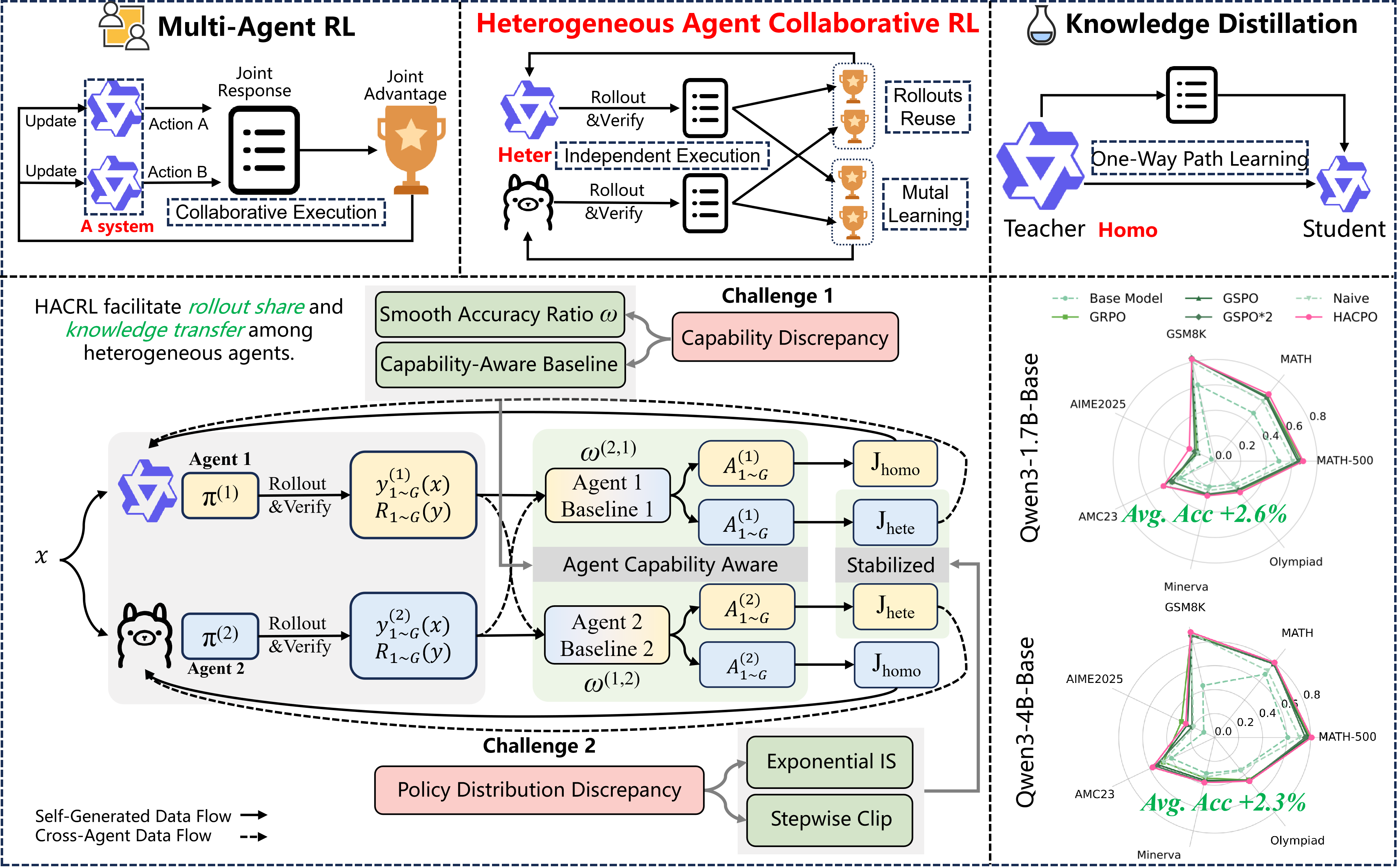}
    \caption{In HACPO, shared rollouts from multiple heterogeneous agents are leveraged for collaborative training. Built upon vanilla RL Optimization, HACPO introduces four algorithmic innovations to mitigate capability and policy distribution discrepancy.}
    \vspace{-0.5cm}
    \label{fig:HACRL}
\end{figure}
\section{Introduction}
Reinforcement Learning with Verifiable Rewards (RLVR) trains reasoning models using automatically checkable reward signals, such as unit tests and formal verifiers \cite{yang2026grouprelativeadvantagebiased,ban2026epistemic}. Group-based algorithms such as GRPO \cite{shao2024deepseekmath} replace the critic in PPO \cite{ppo} with group-relative advantages, motivating variants including DAPO \cite{yu2025dapo} and GSPO \cite{GSPO}.
Despite these advances, RLVR remains bottlenecked by expensive on-policy sampling and verification, which frequently dominate the overall training overhead and limit scalability.

Our intended setting is a \emph{heterogeneous model portfolio}: multiple models serve a shared task distribution under different latency, hardware, cost, and answer-quality requirements. For example, a mathematical problem-solving platform may maintain a small model for latency-sensitive requests and a larger model for requests requiring higher accuracy. Separate client deployments may also use different model families \cite{yang2025qwen3,llama} to fit their existing hardware and software interfaces. These operating constraints motivate improving every model in the portfolio. \emph{Independent deployment} means that each inference instance invokes a single model, without communication with or access to its training partners.

Under isolated RLVR training, these models repeatedly generate and verify rollouts for the same task, but use them only for self-training. This motivates the question: \emph{can heterogeneous models improve together by reusing one another's verified rollouts while retaining independent deployment?} The aim is mutual benefit across the portfolio, while improving the utilization efficiency of rollout.

We first formalize this setting as \textbf{H}eterogeneous \textbf{A}gent \textbf{C}ollaborative \textbf{R}einforcement \textbf{L}earning (\textbf{HACRL}), which captures collaborative policy optimization among heterogeneous agents that execute independently at inference time.
HACRL differs fundamentally from existing paradigms as illustrated in Figure~\ref{fig:HACRL}:
(1) \textbf{LLM-based Multi-Agent Reinforcement Learning (MARL).} \cite{MARFT}
MARL trains agents to coordinate and jointly solve tasks through interaction within a coupled multi-agent system.
In contrast, HACRL does not require coordinated execution. In many practical scenarios, only a single agent is deployed at inference time; however, we still desire that this agent benefits from knowledge acquired from other agents during training.
(2) \textbf{On-/Off-Policy Distillation.}
Distillation typically follows a one-directional “teacher-to-student” paradigm, often among homogeneous agents.
HACRL instead enables bidirectional \textbf{mutual learning} among \textbf{heterogeneous agents}, where each agent simultaneously acts as both a knowledge provider and a learner.

We then propose \textbf{H}eterogeneous \textbf{A}gent \textbf{C}ollaborative \textbf{P}olicy \textbf{O}ptimization (\textbf{HACPO}) to solve HACRL (Figure~\ref{fig:HACRL}).
Compared to vanilla RL optimization, HACPO improves training in two critical aspects:
\noindent
\textbf{(1) Maximized Sample Utilization.}
In an $n$-agent system, each rollout can be reused up to $n$ times, substantially improving sample efficiency.
\noindent
\textbf{(2) Bidirectional Knowledge Transfer.}
By learning from one another, agents acquire complementary knowledge unavailable through self-learning alone, enabling all agents to break performance bottlenecks.
 
In this work, our contributions can be summarized as:

\textbf{[Problem Definition].} We formulate HACRL as a collaborative policy optimization problem for heterogeneous agents under RLVR, aiming to achieve mutual benefit through cross-agent rollout reuse while controlling distribution shifts caused by heterogeneity.

\textbf{[Algorithm].} We propose HACPO to address this problem, with four modifications: 
\textbf{(1)} Agent-Capability-Aware Advantage Estimation, 
\textbf{(2)} Model Capabilities Discrepancy Coefficient, 
\textbf{(3)} Exponential Importance Sampling, and 
\textbf{(4)} Stepwise Clipping. 
These tailored techniques enable the agents to engage in effective and stable mutual learning.

\textbf{[Performance].} 
We evaluate HACPO across three types of heterogeneity and seven challenging mathematical reasoning benchmarks, demonstrating consistent performance improvements, averaging 3.6\%,  while utilizing only half the rollout cost, compared to GSPO with double rollouts.

\section{Heterogeneous Agent Collaborative Reinforcement Learning}

\subsection{Heterogeneous LLM Agent Taxonomy}
\label{sec: Heterogeneous LLM Agent Taxonomy}
Let $\pi_\theta$ denote a large language model (LLM) agent parameterized by $\theta \in \Theta$, where $\Theta$ specifies the complete parameter space, including architecture, dimensionality, and trainable weights.
Let $V_{\pi}$ denote the output vocabulary of agent $\pi_\theta$.
We consider a collaborative policy optimization setting in which multiple LLM agents are jointly optimized toward a shared objective.

We categorize heterogeneity among  distinct LLM agents into three types: (1) heterogeneous state; (2) heterogeneous size; 
(3) heterogeneous model.

\begin{definition}[\textbf{Heterogeneous State}]
Two LLM agents $\pi_{\theta}^{(1)}$ and $\pi_{\theta}^{(2)}$ are said to exhibit \emph{heterogeneous state} if $\Theta_1 = \Theta_2$ and $\dim(\theta_1) = \dim(\theta_2)$, but $\theta_1 \neq \theta_2$ at the start of collaborative policy optimization.
\end{definition}

\begin{definition}[\textbf{Heterogeneous Size}]
Two LLM agents $\pi_{\theta}^{(1)}$ and $\pi_{\theta}^{(2)}$ are said to exhibit \emph{heterogeneous size} if they belong to the same model family and share the same architectural design principles, but have different parameter dimensionalities, i.e., $\dim(\theta_1) \neq \dim(\theta_2)$, with $\theta_1 \neq \theta_2$ at the start of collaborative policy optimization.
\end{definition}

\begin{definition}[\textbf{Heterogeneous Model}]
\label{definition: Heterogeneous Model}
Given two LLM agents $\pi_{\theta}^{(1)}$ and $\pi_{\theta}^{(2)}$, we define them to exhibit \emph{heterogeneous model} heterogeneity if their model architectures differ (e.g., tokenizer, attention mechanism, or training objective), their parameter spaces and sizes are distinct (i.e., $\Theta_1 \neq \Theta_2$), and their initial parameter instantiations are unique (i.e., $\theta_1 \neq \theta_2$).
\end{definition}

\begin{remark}
This taxonomy represents increasing degrees of heterogeneity: heterogeneous state differs only in optimization state, heterogeneous size introduces capacity mismatch, and heterogeneous model captures architectural and representational divergence. This hierarchy enables a systematic study of collaborative policy optimization among heterogeneous LLM agents.
\end{remark}

\subsection{Problem Formalization}

We consider the Heterogeneous Agent Collaborative Reinforcement Learning (HACRL) framework with $n$ LLM agents.
Each agent $k \in \{1,\dots,n\}$ is associated with a policy $\pi_{\theta}^{(k)}$.
All agents operate on a shared task distribution $\mathcal D$ and exhibit heterogeneity as defined in Section~\ref{sec: Heterogeneous LLM Agent Taxonomy}.

During training step $t$, for a prompt $x \sim \mathcal D$, each agent $k$ independently samples $G$ candidate responses from its policy.
The joint response set is:
\begin{equation}
\mathcal Y^{(k)}_t(x) = \{y^{(k)}_{1}, \dots, y^{(k)}_{G}\} \sim \pi^{(k)}_{\theta}(\cdot \mid x), \quad \mathcal Y_t(x) = \bigcup_{k=1}^n \mathcal Y^{(k)}_t(x).
\end{equation}
Since all agents solve the same task, a shared reward function $R(\cdot)$ is applied to every response.
The joint reward set is:
\begin{equation}
\label{eq:joint_reward}
\mathcal R^{(k)}_t(x) = \{ R(y^{(k)}_{i}) \mid i = 1,\dots,G \}, \quad \mathcal R_t(x) = \bigcup_{k=1}^n \mathcal R^{(k)}_t(x).
\end{equation}

\begin{definition}[\textbf{HACRL Problem}]
Consider a system of $n$ heterogeneous agents.
For a prompt $x \sim \mathcal D$, let $\mathcal Y(x)$ and $\mathcal R(x)$ denote the joint response and reward sets, respectively.
The objective of \emph{Heterogeneous Agent Collaborative Reinforcement Learning} is to optimize each agent
$k \in \{1,\dots,n\}$ by maximizing
\begin{equation}
J^{(k)}
=
J_{\mathrm{homo}}^{(k)}\!\left(Y^{(k)}_t(x), \mathcal R^{(k)}_t(x)\right)
+
J_{\mathrm{hete}}^{(k)}\!\left(\{Y^{(j)}_t(x), \mathcal R^{(j)}_t(x)\}_{j \neq k}\right),
\end{equation}
where $J_{\mathrm{homo}}^{(k)}$ is computed using rollouts generated by agent $k$ itself, and
$J_{\mathrm{hete}}^{(k)}$ leverages rollouts generated by the other agents.
\end{definition}

This formulation enables each agent to benefit from both self-generated experiences and cross-agent information under collaborative reinforcement learning.

\section{Heterogeneous Agent Collaborative Policy Optimization}
\label{sec: method}
In this section, we propose HACPO, a novel multi-agent collaborative optimization algorithm (procedure is shown in Appendix \ref{Overall of HACPO}): for \emph{one} given task, \emph{multiple} heterogeneous LLM agents execute independently and learn from each other.

Facing two challenges shown in Figure \ref{fig:HACRL}: the discrepancy of \emph{agent capability} and \emph{policy distribution}, Our method addresses the above challenges through the following components: (1) Agent-Capability-Aware Advantage Estimation; (2) Model Capability Discrepancy Coefficient; (3) Exponential Importance Sampling; (4) Stepwise Clipping. 




    

\subsection{Agent-Capability-Aware Advantage Estimation}
\label{sec:capability_aware_adv}
At training step $t$, for each prompt $x$, each agent $k \in \{1,...,n\}$ generates $G$ responses
$\{y_{t,i}^{(k)}\}_{i=1}^G \sim \pi^{(k)}_{\theta}(\cdot \mid x)$.
For a single agent, the standard group-relative advantage estimator \cite{shao2024deepseekmath} is:
\begin{equation}
\label{eq:single_agent_adv}
A_{\text{single}.t,i}^{(k)}(y_{t,i}^{(k)})
=
\frac{R\!\left(y_{t,i}^{(k)}\right) - \frac{1}{G} \sum_{i=1}^G R\!\left(y_{t,i}^{(k)}\right)}{std \{ {\mathcal{R}^{(k)}_{t}(x)} \}}.
\end{equation}

While Eq.~\eqref{eq:single_agent_adv} is appropriate for training a single model in isolation,
it becomes suboptimal in a multi-agent settings  where agents exhibit heterogeneous capabilities. 
Relying solely on self-generated responses fails to leverage valuable information from other agents, 
while \emph{naively averaging rewards across all agents disregards inter-model capability differences 
and often results in miscalibrated advantage estimates}.

To address this issue, we propose an \emph{agent-capability-aware} advantage estimator.
The advantage of response $y_{t,i}^{(k)}$ for agent $k$ is defined as
\begin{equation}
\label{eq:capability_aware_adv}
A_{t,i}^{(k)}\left(y_{t,i}^{(k)}\right)
=
\frac{R\!\left(y_{t,i}^{(k)}\right) - \mu_t^{(k)}}{std \{ {\mathcal{R}_{t}(x)} \}}, 
\quad \mu_t^{(k)} = \frac{1}{nG}
\sum_{j=1}^n
\sum_{i=1}^G
\omega_t^{(k,j)} \, R\!\left(y_{t,i}^{(j)}\right), 
\end{equation}
Here, $\mu_t^{(k)}$ is the capability-adjusted baseline.  $\omega_t^{(k,j)}$ is a \emph{capability ratio} that rescales responses from agent $j$  when estimating the baseline for agent $k$, defined as:
\begin{equation}
\label{eq:capability_ratio}
\omega_t^{(k,j)} = \frac{P_t^{(k)}}{P_t^{(j)}}, \quad
P_t^{(k)} = \frac{1}{|\mathcal{B}|G}\sum_{x \in  \mathcal{B}_t} \sum_{i=1}^G R\left(y_{t,i}^{(k)}\right),
\end{equation}
where the $P_t^{(k)}$ denotes an estimate of the performance of agent $k$, obtained by averaging the mean rewards of the batch $\mathcal{B}_t$ at the current step $t$.

Intuitively, when estimating the advantage baseline in a group for agent $k$, rewards from other agents are
reweighted according to their relative capabilities, allowing all responses to contribute while preserving agent-specific calibration. 

We further analyze the capability-aware adjustment in Appendix~\ref{sec: Theoretical Analysis}. 
Since the normalization factor $\mathrm{std}\{\mathcal{R}_{t}(x)\}$ is known to introduce a common but minor bias in practice, our analysis focuses on the centered reward 
$\bar{A}_{t,i}^{(k)}(y)=R(y)-\mu_t^{(k)}$.

As shown in Appendix~\ref{sec:oracle_capability_ratio_concentration}, the oracle capability-aware baseline is exactly unbiased: its expectation matches the agent's true expected reward, so the corresponding centered reward has zero mean. The Appendix \ref{sec: Empirical Ratio Concentration} further analyzes the practical batch-level empirical ratio and proves uniform high-probability bounds on its estimation error via Hoeffding's inequality and union bounds.

This theoretical guarantee ensures that HACPO can safely incorporate heterogeneous cross-agent rollouts to enrich the learning signal and maximize sample efficiency—while introducing no systematic bias in the oracle case, and only bounded, controllable error in the finite-batch implementation.

\subsection{Model Capabilities Discrepancy  Coefficient}
\label{sec: Model Capabilities Discrepancy  Coefficient}

To address capability discrepancies across heterogeneous agents, we employ the capability ratio $\omega_t^{(k,j)}$, introduced earlier, as a quantitative measure of relative model competence.
When training agent $k$, advantages computed from samples generated by other
agents are rescaled according to their relative capability. This design encourages an
agent to learn more aggressively from stronger agents, while adopting a more conservative
update when incorporating samples from weaker ones.

Formally, suppose that agent $k$ is updated at training step $t$ using a response
$y_{t,i}^{(j)}$ generated by agent $j$. The effective advantage used for updating agent
$k$ is defined as
\begin{equation}
\label{eq:capability_scaled_adv}
\tilde{A}_{t,i}^{(k)} \left(y_{t,i}^{(j)}\right)=
\begin{cases}
A_{t,i}^{(k)} \left(y_{t,i}^{(j)}\right)
& j = k \\
\omega_t^{(j,k)} \cdot A_{t,i}^{(k)}\left(y_{t,i}^{(j)}\right)
& j \neq k
\end{cases}
\end{equation}
Here, $\omega_t^{(k,j)}$ represents the performance ratio between agents $k$ and $j$ at
training step $t$, with larger values indicating that agent $k$ outperforms agent $j$.

The capability ratio $\omega_t^{(k,j)}$ serves two complementary roles to enable stable collaboration: 

\textbf{(1) Baseline Calibration:} In Sec.~\ref{sec:capability_aware_adv}, it rescales cross-agent reward statistics to properly calibrate the baseline $\mu_t^{(k)}$. 

\textbf{(2) Gradient Modulation:} In Eq.~\eqref{eq:capability_scaled_adv}, it acts as a modulation factor that amplifies learning signals from stronger agents while attenuating those from weaker ones.

\subsection{Exponential Importance Sampling}
\label{Exp IS}

Importance sampling is commonly used to correct distributional mismatches between samples generated by different policies. Following GSPO, we adopt a sequence-level importance ratio and extend it to the heterogeneous multi-agent setting. When updating agent $k$ at
step $t$, for a response $y_{t,i}^{(j)}$ generated by agent $j$, we define
\begin{equation}
s_{t,i}^{(k,j)}
=
\frac{\pi^{(k)}_{\theta}\!\left(y_{t,i}^{(j)}\right)^{\frac{1}{L_k}} }
{\pi_{\theta_{\mathrm{old}}}^{(j)}\!\left(y_{t,i}^{(j)}\right)^{\frac{1}{L_j}} }
.
\end{equation}
Here, $L_k$ and $L_j$ respectively represent the length of response $y_{t,i}^{(j)}$ under tokenizers of agent k and j.  
For combinations of heterogeneous agents that satisfy Definition~\ref{definition: Heterogeneous Model} with incompatible tokenizers, we detokenize the response into text and retokenize it using the target agent's tokenizer. Through sequence-level normalization, the slight length discrepancies arising from re-tokenization become negligible. The experiment shows that on the MATH training set, the tokenizers of Qwen and Llama differ in the number of tokens produced for the same reasoning (700 tokens average) by only 4\%.

In heterogeneous settings, inter-agent policy discrepancies can be much larger than
on-policy updates, making direct use of this ratio overly aggressive. To mitigate this
issue, we introduce a non-gradient exponential reweighting:
\begin{equation}
\label{eq: hete IS}
\tilde{s}_{t,i}^{(k,j)}
=
s_{t,i}^{(k,j)} \cdot \bigl(\mathrm{sg}[\,s_{t,i}^{(k,j)}\,]\bigr)^{\alpha} \quad  k \neq j, s_{t,i}^{(k,j)}<1.0 
\end{equation}
where $\mathrm{sg}[\cdot]$ denotes the stop-gradient operator and $\alpha \ge 0$ controls
the degree of conservativeness.

This design biases agent $k$ toward learning from agents whose output distributions are more aligned with its own, while reducing the impact of large cross-agent distribution shifts.

\subsection{Stepwise Clipping}

The cross-agent importance sampling ratio $s^{(k,j)}_{t,i}$ exhibits fundamentally different behaviors from the self-agent ratio $s^{(k,k)}_{t,i}$, necessitating tailored clipping mechanisms (see Appendix~\ref{Heterogeneous Agent Importance Sampling Analysis} for further empirical details):

\textbf{Asymmetric Clipping Bounds.} Across training iterations, $s^{(k,k)}_{t,i}$ typically fluctuates closely around $1.0$, requiring a narrow clipping range (e.g., $[0.9997, 1.0004]$). In contrast, $s^{(k,j)}_{t,i}$ evolves dynamically and remains significantly smaller than $1.0$. To maintain a reasonable clipping proportion, we adopt an asymmetric clipping range for the cross-agent setting, which is typically bounded within $[0.8, 1.0]$.

\textbf{Stepwise Clipping.} Within a single training batch, $s^{(k,k)}_{t,i}$ decays as parameter updates increase, which naturally elevates the clipping proportion. However, $s^{(k,j)}_{t,i}$ fluctuates irregularly. To prevent the later updates within a batch from being dominated by cross-agent responses, we introduce a stepwise strategy that gradually tightens the cross-agent clipping bounds as updates progress within one step.

Let $m$ denote the number of parameter updates performed so far within the current step, and $\delta_{\mathrm{step}}$ denote the per-update tightening factor. Combining these principles, the clipping function for cross-agent importance sampling is formulated as:
\begin{equation}
\mathrm{clip}\!\left(
s^{(k,j)}_{t,i},\,
1 - \delta + m \cdot \delta_{\mathrm{step}},\,
1.0
\right), 
\end{equation}
where the lower-bound hyperparameter $\delta$ is empirically initialized to $0.2$ and is dynamically within a step reduced following our stepwise strategy.

\section{Experiment}
\label{sec: Experiment}
\textbf{Setting Details.}
We adopt 7.5k high quality math questions from the MATH dataset \cite{math} for training.
During evaluation, we select a comprehensive set of benchmarks: MATH-500, MATH, GSM8K \cite{gsm8k}, AIME2025, AMC23, Minerva\cite{minerva} and Olympiad\cite{olympiadbench}.

To verify the effectiveness of our method, we conduct experiments on the three heterogeneity settings mentioned in Section \ref{sec: Heterogeneous LLM Agent Taxonomy}.
We compare our approach against the following baselines: (1) \textbf{Standard Single-Agent Baselines (GRPO, GSPO)}, which serve as benchmarks for isolated training performance (same rollout cost as HACPO but with half the policy updates); (2) \textbf{Data-Equivalent Baseline (GSPO$\times$2)}, a single-agent GSPO setting with double rollouts and updates in every step. This serves to rule out the impact of increased data and updates, verifying the complementary value of heterogeneous agents (double the rollout cost of HACPO but with the same policy updates); (3) \textbf{Naive Collaborative Baseline (Naive)}, a two-agent setting with shared rollouts but lacking the algorithmic innovations in Section \ref{sec: method}, used to validate the necessity of our proposed discrepancy mitigation techniques (same rollout and policy update costs as HACPO).

We focus our empirical comparisons on different RLVR baselines, omitting Knowledge Distillation (KD) and Multi-Agent RL (MARL) due to fundamental differences in problem formulation. Specifically: (1) Unlike KD, which is a unidirectional process relying on a frozen, homogeneous teacher's output distribution, HACRL enables bidirectional mutual improvement across heterogeneous agents (regardless of relative capabilities). (2) Unlike MARL, which typically requires coupled agents for joint execution, our agents maintain strictly independent execution at inference time.
\begin{table}[h!]
\scriptsize
\caption{Main results across three heterogeneity settings (state, size and model). We compare our method against Standard Single-Agent Baselines (GRPO, GSPO), a Resource-Equivalent Baseline (GSPO$\times$2) and a Naive multi-agent rollout share baseline(Naive). The content in the brackets represents a comparison with GSPO $\times$ 2.}
\label{Main Results}
\begin{center}
\begin{tabular}{lccccccccr}
\toprule
Model & MATH-500 & MATH & GSM8K & AIME2025 & AMC23 & Minerva & Olympiad & AVG\\
\midrule
\multicolumn{9}{c}{Qwen3-4B and Qwen3-4B-Instruct (\textbf{Heterogeneous State})}  \\
\midrule
4B                 & 0.802 & 0.836 & 0.907 & 0.335 & 0.65  & 0.39  & 0.524 & 0.635 \\
4B (GRPO)           & 0.88  & 0.889 & 0.918 & 0.582 & 0.775 & 0.386 & 0.592 & 0.717\\
4B (GSPO)           & 0.854 & 0.87  & 0.925 & 0.485 & 0.675 & 0.412 & 0.564 & 0.684 \\
4B (GSPO$\times$2)         & 0.876 & 0.875 & 0.923 & 0.522 & 0.675 & 0.39 & 0.579 & 0.691 \\
4B (Naive)          & 0.728 & 0.737 & 0.891 & 0.378 & 0.6   & 0.353 & 0.394 & 0.583 \\
\rowcolor{gray!20} 4B(HACPO) & \textbf{0.91}  & \textbf{0.905} & \textbf{0.933} & \textbf{0.622} & \textbf{0.85}  & \textbf{0.423} & \textbf{0.643} & \textbf{0.755 (+0.064)}\\
4B-Instruct        & 0.938 & 0.937 & 0.936 & 0.696 & 0.85  & 0.441 & 0.722 & 0.789 \\
4B-Instruct (GRPO)  & 0.93  & 0.933 & 0.933 & 0.676 & 0.875 & 0.43  & 0.72 & 0.785 \\
4B-Instruct (GSPO)  & 0.938 & 0.94  & 0.939 & 0.72  & 0.9   & 0.43  & 0.726 & 0.799 \\
4B-Instruct (GSPO$\times$2)  & 0.932 & 0.939  & 0.942 & 0.74  & 0.9   & 0.43  & 0.711 & 0.799 \\
4B-Instruct(Naive) & 0.844 & 0.845 & 0.936 & 0.547 & 0.725 & 0.39  & 0.552 & 0.691 \\
\rowcolor{gray!20} 4B-Instruct(HACPO)   & \textbf{0.948} & \textbf{0.943} & \textbf{0.946} & \textbf{0.757} & \textbf{0.95}  & \textbf{0.452} & \textbf{0.732} & \textbf{0.818 (+0.019)}\\
\midrule
\multicolumn{9}{c}{Qwen3-1.7B-Base and Qwen3-4B-Base  (\textbf{Heterogeneous Size})}  \\
\midrule
1.7B-Base         & 0.5   & 0.483 & 0.616 & 0.033 & 0.3   & 0.206 & 0.229 & 0.338 \\
1.7B-Base (GRPO)   & 0.682 & 0.652 & 0.824 & 0.16 & 0.375 & 0.272 & 0.298 & 0.466 \\
1.7B-Base (GSPO)   & 0.648 & 0.641 & 0.826 & 0.148 & 0.45  & 0.272 & 0.287 & 0.467 \\
1.7B-Base (GSPO$\times$2)   & 0.664 & 0.65 & \textbf{0.829} & 0.177 & 0.375  & 0.265 & 0.293 & 0.465 \\
1.7B-Base (Naive)  & 0.608 & 0.601 & 0.798 & 0.147 & 0.325 & 0.235 & 0.263 & 0.425 \\
\rowcolor{gray!20} 1.7B-Base(HACPO)    & \textbf{0.69}  & \textbf{0.674} & 0.822 & \textbf{0.225} & \textbf{0.45}  & \textbf{0.279} & \textbf{0.314} & \textbf{0.493 (+0.028)}\\
4B-Base           & 0.61  & 0.676 & 0.445 & 0.1   & 0.4   & 0.308 & 0.347 & 0.412 \\
4B-Base (GRPO)     & 0.796 & 0.788 & 0.885 & \textbf{0.307} & 0.475 & 0.349 & 0.454 & 0.579 \\
4B-Base (GSPO)     & 0.782 & 0.787 & 0.877 & 0.25  & 0.525 & 0.368 & 0.46  & 0.578 \\
4B-Base (GSPO$\times$2)     & 0.756 & 0.794 & 0.873 & 0.208  & 0.55 & 0.382 & 0.463  & 0.575 \\
4B-Base (Naive)    & 0.708 & 0.712 & 0.895 & 0.196 & 0.475 & 0.342 & 0.354 & 0.526 \\
\rowcolor{gray!20} 4B-Base (HACPO) & \textbf{0.808} & \textbf{0.801} & \textbf{0.903} & 0.267 & \textbf{0.575} & \textbf{0.386} & \textbf{0.467} & \textbf{0.601 (+0.026)}\\
\midrule
\multicolumn{9}{c}{Qwen3-4B-Base and Llama3.2-3B-Instruct  (\textbf{Heterogeneous Model})}  \\
\midrule
Qwen3             & 0.61  & 0.676 & 0.445 & 0.1   & 0.4   & 0.308 & 0.347 & 0.412 \\
Qwen3 (GRPO)     & \textbf{0.796} & 0.788 & 0.885 & \textbf{0.307} & 0.475 & 0.349 & 0.454 & 0.579 \\
Qwen3 (GSPO)       & 0.782 & 0.787 & 0.877 & 0.25  & 0.525 & 0.368 & \textbf{0.46} & 0.578 \\
Qwen3 (GSPO$\times$2)     & 0.756 & \textbf{0.794} & 0.873 & 0.208  & 0.55 & \textbf{0.382} & 0.463  & 0.575 \\
Qwen3 (Naive)      & 0.734 & 0.712 & 0.895 & 0.143 & 0.55  & 0.342 & 0.354 & 0.526 \\
\rowcolor{gray!20} Qwen3 (HACPO) & 0.786 & 0.783 & \textbf{0.921} & 0.268 & \textbf{0.6}   & 0.379 & 0.442 & \textbf{0.597 (+0.022)}\\
Llama3.2          & 0.267 & 0.441 & 0.788 & 0.0   & 0.2   & 0.169 & 0.158 & 0.289 \\
Llama3.2 (GRPO)   & 0.502 & 0.507 & 0.814 & 0.0 & 0.25 & \textbf{0.199} & 0.174 & 0.349 \\
Llama3.2 (GSPO)   & 0.512 & 0.501 & 0.812 & 0.054 & 0.225 & 0.184 & 0.17 & 0.351 \\
Llama3.2 (GSPO$\times$2)   & 0.488 & 0.498 & \textbf{0.829} & 0.0 & 0.175 & 0.188 & 0.159 & 0.334 \\
Llama3.2 (Naive)  & 0.406 & 0.407 & 0.734 & 0.0   & 0.225 & 0.177 & 0.107 & 0.294 \\
\rowcolor{gray!20} Llama3.2 (HACPO)  & \textbf{0.566} & \textbf{0.548} & 0.826 & \textbf{0.054} & \textbf{0.35}  & 0.176 & \textbf{0.208} & \textbf{0.39 (+0.056)}\\
\bottomrule
\end{tabular}
\end{center}
\vspace{-0.2cm}
\end{table}
\begin{figure}[h!]
  \centering
  \captionsetup[subfloat]{font=tiny}
  \subfloat[Qwen3-4B and Qwen3-4B-Instruct\label{fig:group1}]{
    \begin{minipage}{0.32\linewidth}
      \includegraphics[width=\linewidth]{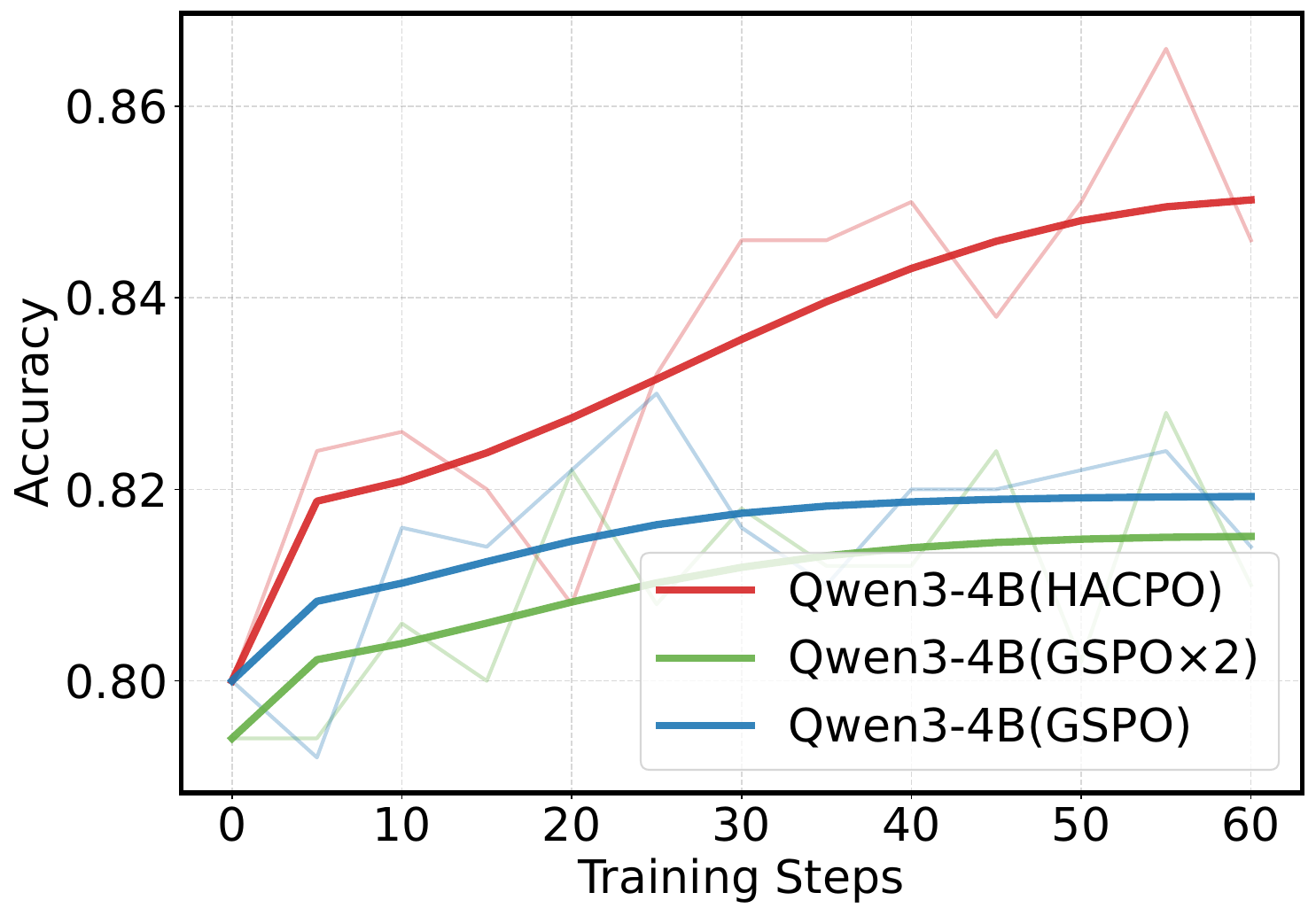}
      \\[0.2cm]
      \includegraphics[width=\linewidth]{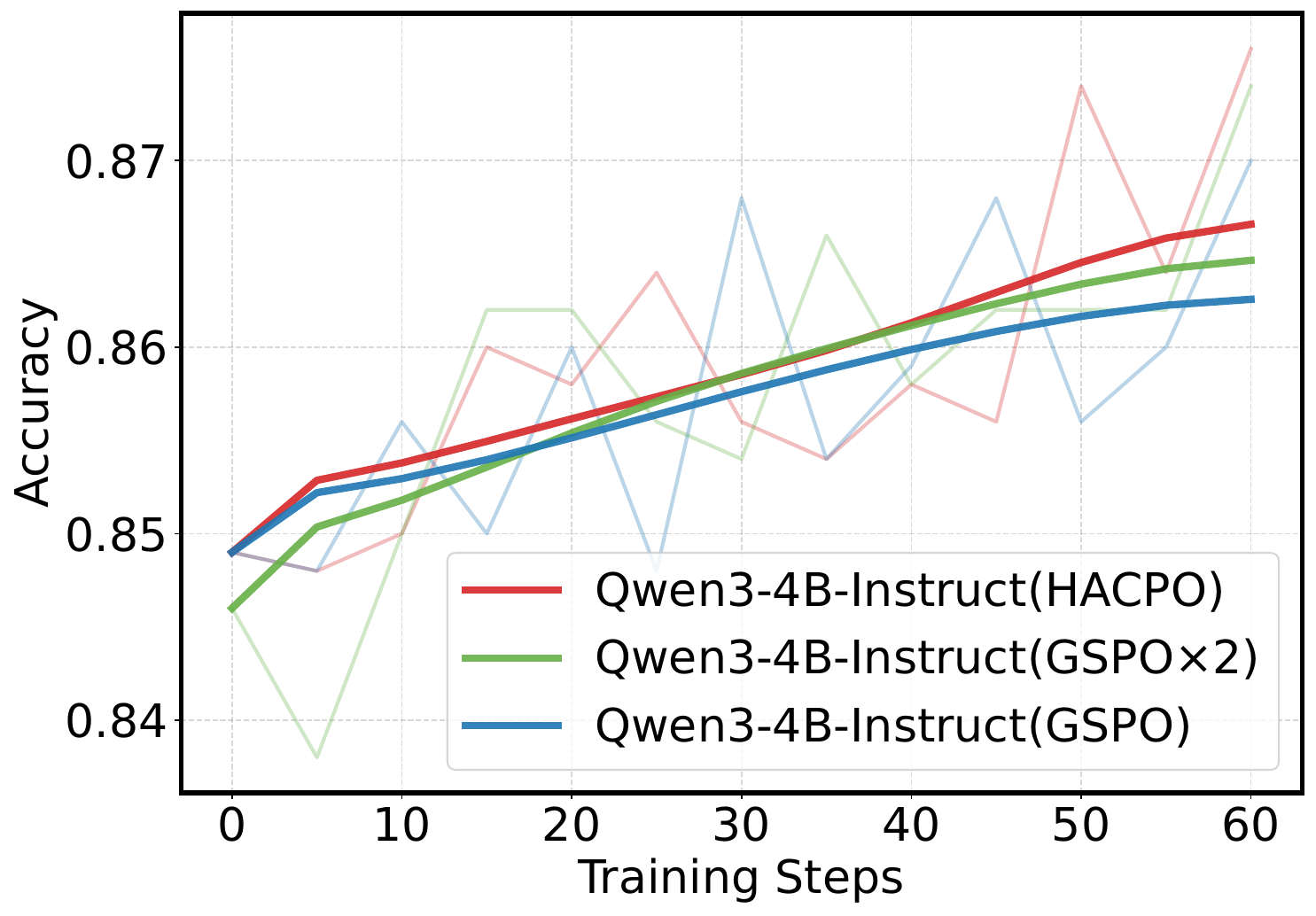}
    \end{minipage}
  }
  \hfill
  \subfloat[Qwen3-1.7B-Base and Qwen3-4B-Base\label{fig:group2}]{
    \begin{minipage}{0.32\linewidth}
      \includegraphics[width=\linewidth]{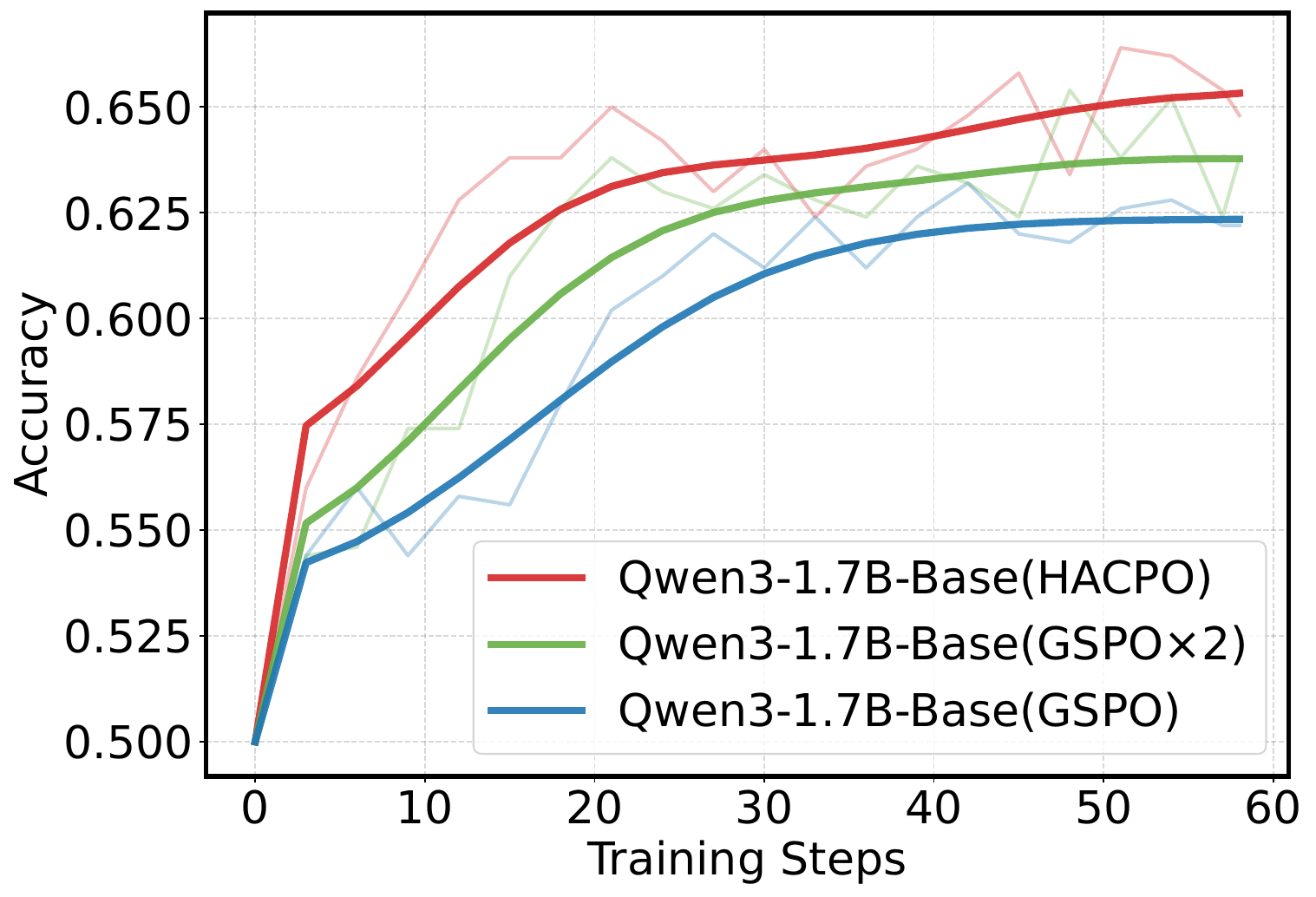}
      \\[0.2cm]
      \includegraphics[width=\linewidth]{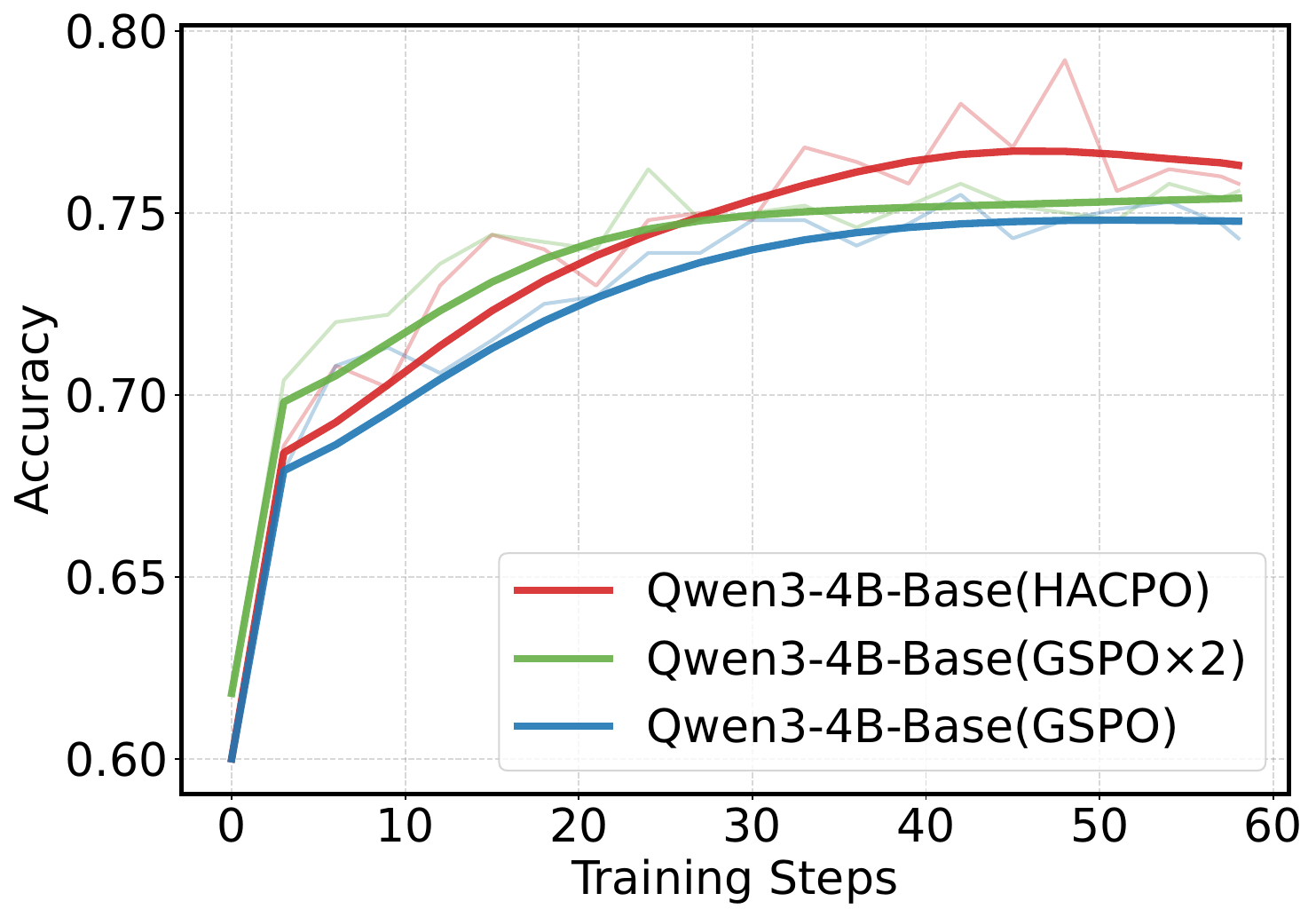}
    \end{minipage}
  }
  \hfill
  \subfloat[Qwen3-4B-Base and Llama3.2-3B-Instruct\label{fig:group3}]{
    \begin{minipage}{0.32\linewidth}
      \includegraphics[width=\linewidth]{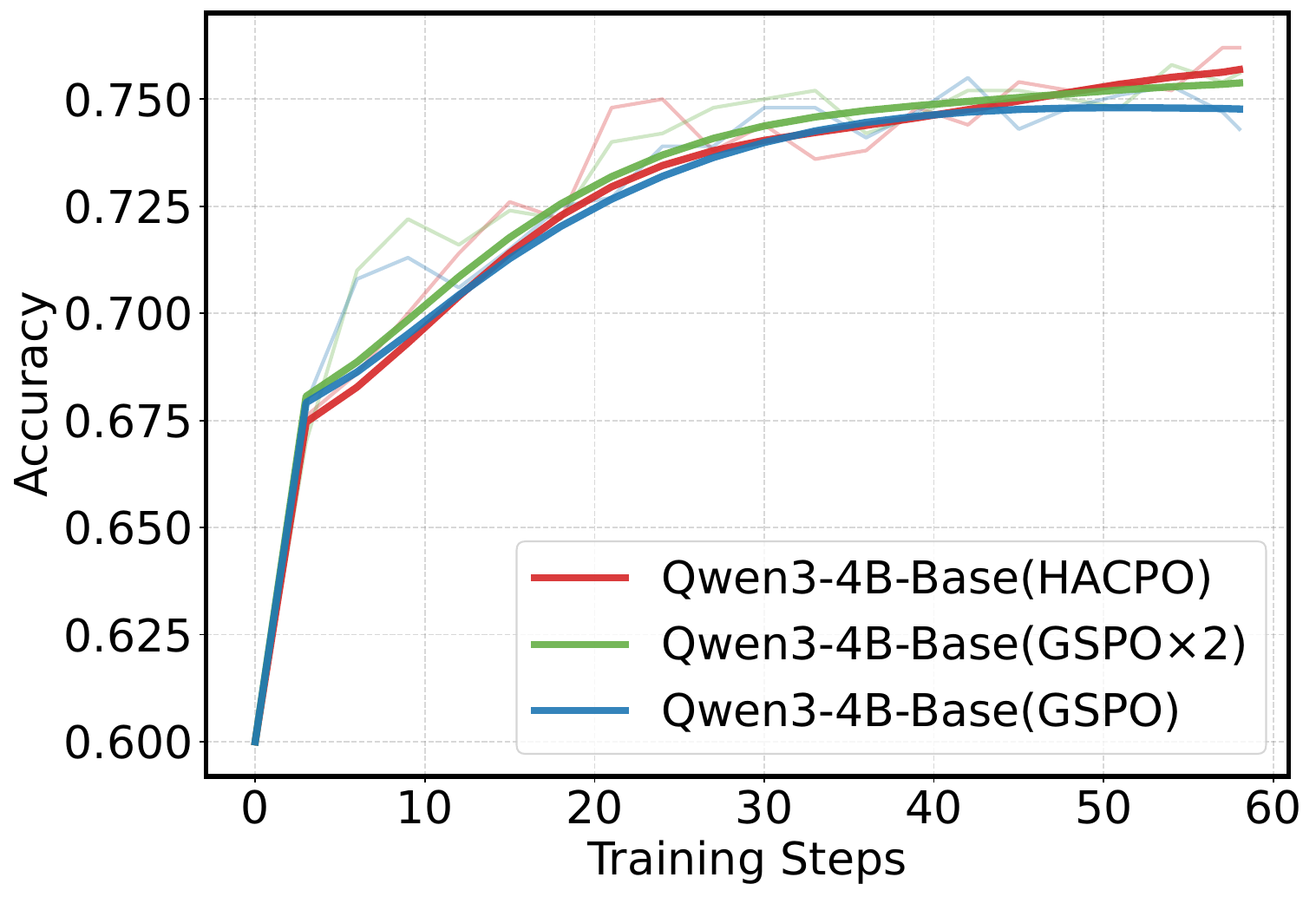}
      \\[0.2cm]
      \includegraphics[width=\linewidth]{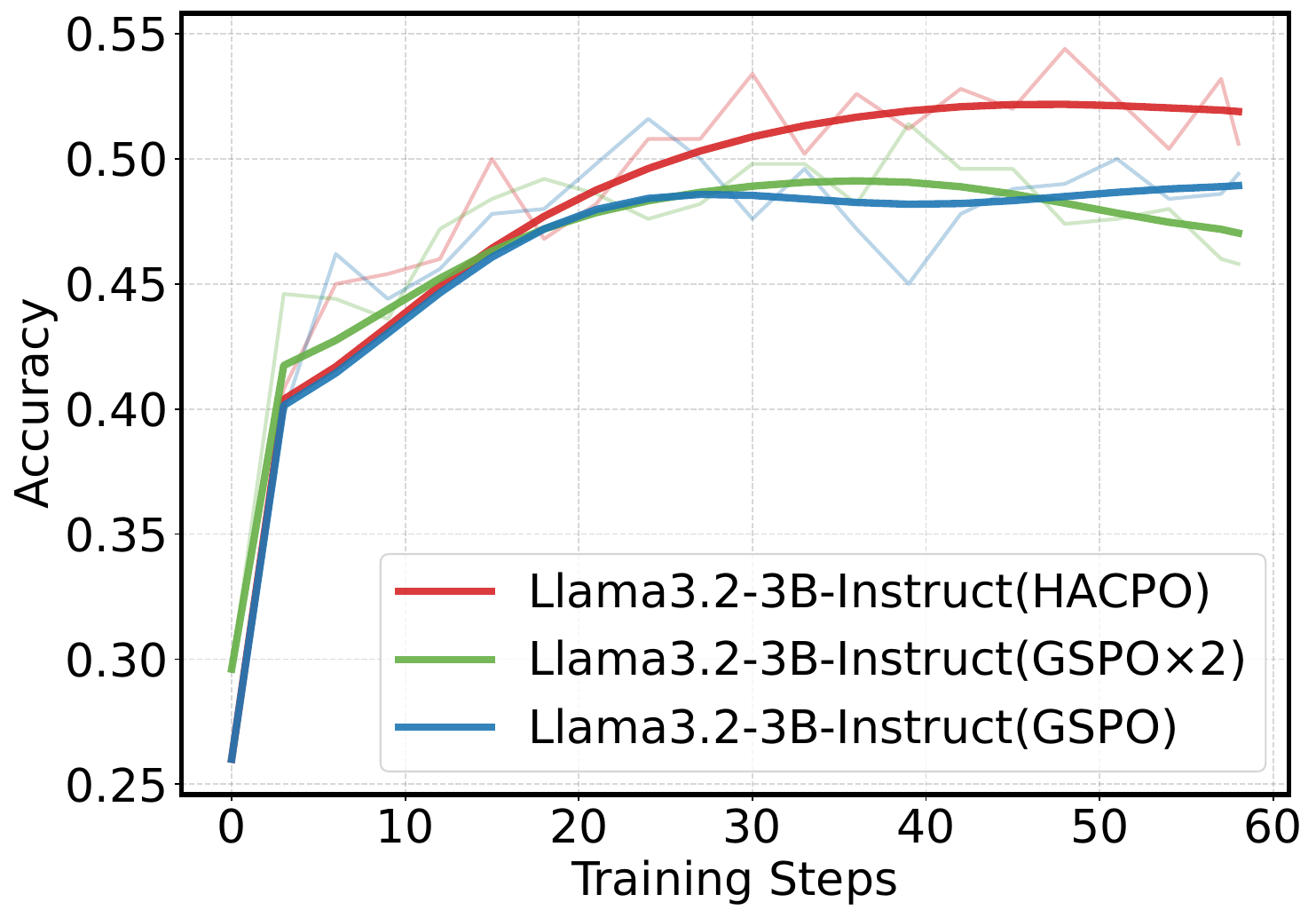}
    \end{minipage}
  }
  \caption{Training Curves of HACPO, GSPO and GSPO $\times$ 2}
  \label{fig: training_curve}
  \vspace{-0.5cm}
\end{figure}
\subsection{Result and Analysis}
As detailed in Table \ref{Main Results}, HACPO demonstrates superior final performance compared to all baselines across various heterogeneous settings. Across 7 benchmarks, HACPO achieves an average accuracy improvement of +3.6\% over the strong baseline GSPO×2, while requiring only half the rollout cost. To illustrate the learning dynamics, Figure \ref{fig: training_curve} presents the training curves of HACPO versus the single-agent GSPO and GSPO $\times$ 2 baseline. We attribute these performance gains to two primary mechanisms inherent in the HACPO: (1) Capability-driven guidance, where stronger models assist in enhancing the performance of weaker ones; and (2) Mutual knowledge exchange, which involves the sharing of complementary rollouts—encompassing both correct solutions and informative errors—between agents.

The improvements are consistent across model combinations, benchmarks, and random seeds, confirming that HACPO's gains are robust. Specifically, we conduct reproducibility experiments with five random seeds on the Qwen3-1.7B-Base and Qwen3-4B-Base combination. Across all runs, HACPO consistently outperforms GSPO by a stable margin: +3.58\% on average for Qwen3-1.7B-Base and +2.86\% for Qwen3-4B-Base on MATH500. Full results are provided in Appendix \ref{sec: The Performance over Different Seeds}. In total, we evaluate HACPO on 6 heterogeneous agent combinations. The main text presents 3 representative settings; results for the additional 3 combinations are deferred to Appendix \ref{sec:Three More Model Combinations}. 

We also report peak GPU memory utilization and wall-clock runtime to demonstrate HACPO's more efficient use of rollout compute. Detailed profiling results are included in Appendix \ref{sec: The GPU Peak Memory and Overall Runtime}.


\textbf{Heterogeneous State.}  
In the Qwen3-4B and Qwen3-4B-Instruct setting, we observe asymmetric but non-trivial gains: while the 4B model improves more substantially, the Instruct model also exhibits consistent performance improvements.
Although this setting corresponds to heterogeneous state, where agents differ only due to post-training stages, HACPO still enables the stronger agent to benefit from the weaker one. Specifically, the weaker agent contributes complementary exploration signals—such as alternative reasoning paths and informative errors—that are underrepresented in the stronger agent’s own rollouts.
As a result, learning is not purely unidirectional. Even when capability-driven guidance dominates, the stronger agent can still extract useful supervisory signals from the weaker agent, leading to measurable performance gains.

\textbf{Heterogeneous Size.}  
In the Qwen3-1.7B-Base and Qwen3-4B-Base setting, both models improve significantly, validating the mechanism of mutual knowledge exchange. Even with lower capability, the 1.7B model serves as a distinct explorer, generating valuable erroneous responses and a few unique correct solutions that the 4B model fails to produce, thereby facilitating bidirectional knowledge transfer. 
Rollout analysis further shows that the weaker model attains a higher sampled success rate on 8.4\% (627/7424) of the analyzed problems (Appendix~\ref{sec:weak_to_strong}).

\textbf{Heterogeneous Model.}  
Finally, we consider the heterogeneous model setting involving Qwen3-4B-Base and Llama3.2-3B-Instruct, which differ substantially in architecture, tokenizer, and training objectives. Despite this high degree of heterogeneity, we observe consistent performance improvements in both models.
These results demonstrate that HACPO is able to extract transferable knowledge from cross-model rollouts and effectively share it across heterogeneous agents. By leveraging verified responses—including correct solutions and informative failure cases—each model can learn from complementary reasoning patterns that are absent from its own policy distribution.


The experimental results show that HACPO significantly improves performance across all three types of heterogeneity, validating its generality and robustness. Additionally, the differences observed among the three settings shed light on the two underlying mechanisms of HACPO.

\textbf{Beyond One-Way Transfer.}
For the 1.7B/4B-Base pair, the small model's average score increases from 0.406/0.408 with OPD/OPD+GRPO to 0.452 with HACPO (Frozen Strong), and to 0.493 with full HACPO (Appendix~\ref{sec:distillation_baselines}). Joint optimization therefore provides benefits beyond a frozen partner in this setting. Compared with first training the large model and then distilling the small model, HACPO improves their average scores by 2.6 and 3.3 percentage points, respectively, while reducing the time to obtain both models by 14.2\% (Appendix~\ref{sec:train_then_distill}). Together with the rollout analysis, these results support bidirectional learning within an independently deployable model portfolio.

\subsection{Three-Agent Collaboration}
\label{sec:three_agent}
We jointly train Qwen3-1.7B-Base, Llama3.2-1B-Instruct, and Gemma3-1B-Instruct~\cite{gemma2025gemma3}, extending collaboration to three distinct model families. We compare each resulting model with independently trained GSPO and GSPO$\times$3. The latter triples the single-agent rollout budget to match the number of training rollouts available to each HACPO agent from the shared pool.

\begin{table}[t]
\centering
\scriptsize
\setlength{\tabcolsep}{4pt}
\caption{Three-agent collaboration across model families. AVG is the mean accuracy across seven benchmarks; parentheses report absolute gains over GSPO$\times$3.}
\label{tab:three_agent}
\begin{tabular}{lcccccccc}
\toprule
Method & MATH500 & MATH & GSM8K & AIME2025 & AMC23 & Minerva & Olympiad & AVG \\
\midrule
\multicolumn{9}{c}{Qwen3-1.7B-Base} \\
\midrule
GSPO & 0.648 & 0.641 & 0.826 & 0.148 & \textbf{0.450} & 0.272 & 0.287 & 0.467 \\
GSPO$\times$3 & 0.664 & 0.668 & \textbf{0.833} & 0.133 & 0.400 & \textbf{0.302} & 0.296 & 0.471 \\
\rowcolor{gray!20} HACPO & \textbf{0.690} & \textbf{0.688} & 0.824 & \textbf{0.177} & 0.425 & 0.287 & \textbf{0.317} & \textbf{0.487 (+0.016)} \\
\midrule
\multicolumn{9}{c}{Llama3.2-1B-Instruct} \\
\midrule
GSPO & \textbf{0.356} & 0.346 & 0.523 & 0.021 & 0.125 & 0.066 & \textbf{0.088} & 0.218 \\
GSPO$\times$3 & 0.340 & 0.354 & 0.520 & 0.000 & 0.175 & 0.077 & 0.077 & 0.220 \\
\rowcolor{gray!20} HACPO & 0.340 & \textbf{0.363} & \textbf{0.563} & \textbf{0.033} & \textbf{0.200} & \textbf{0.081} & \textbf{0.088} & \textbf{0.238 (+0.018)} \\
\midrule
\multicolumn{9}{c}{Gemma3-1B-Instruct} \\
\midrule
GSPO & 0.452 & 0.451 & 0.616 & 0.033 & 0.225 & 0.052 & 0.150 & 0.283 \\
GSPO$\times$3 & 0.472 & 0.450 & 0.623 & 0.033 & 0.225 & 0.040 & \textbf{0.167} & 0.287 \\
\rowcolor{gray!20} HACPO & \textbf{0.496} & \textbf{0.463} & \textbf{0.626} & \textbf{0.067} & \textbf{0.275} & \textbf{0.059} & 0.158 & \textbf{0.306 (+0.019)} \\
\bottomrule
\end{tabular}
\end{table}

As shown in Table~\ref{tab:three_agent}, HACPO improves the average scores of all three models over GSPO$\times$3 by 1.6, 1.8, and 1.9 percentage points, respectively. These results demonstrate effective collaboration beyond two agents, including partners from three different families. Gains vary across individual benchmarks, while all three models benefit on average.

\subsection{Ablation Study}
\label{sec: Ablation Experiment}
\textbf{Agent-Capability-Aware Advantage Estimation.} Ablation on the Qwen3-1.7B/4B-Base combination (Table \ref{Ablation of Advantage Estimator}) confirms that removing this module significantly degrades performance. This decline stems from the systematic bias in standard group-relative advantages in multi-agent setting due to the capability discrepancy cross heterogenous agents. Our method addresses this by constructing \emph{agent-capability-aware} advantage baselines—raising the standard for the stronger models and lowering it for the weaker ones—thereby preserving baseline unbiasedness in the oracle setting of Theorem~\ref{thm:oracle_unbiased_baseline}.
\begin{table*}[h!]
\scriptsize
\caption{Ablation of Advantage Estimator}
\label{Ablation of Advantage Estimator}
\vspace{-0.3cm}
\begin{center}
\begin{tabular}{lcccccccccr}
\toprule
Model & MATH-500 & MATH & GSM8K & AIME2025 & AMC23 & Minerva & Olympiad & AVG\\
\midrule
1.7B(HACPO - Adv)     & \textbf{0.696} & 0.659 & \textbf{0.825} & 0.126 & 0.375 & 0.261 & 0.313 & 0.465\\
1.7B(HACPO)  & 0.69  & \textbf{0.674} & 0.822 & \textbf{0.225} & \textbf{0.45}  & \textbf{0.279} & \textbf{0.314} & \textbf{0.493}\\
4B(HACPO - Adv)      & 0.774 & 0.771 & \textbf{0.912} & \textbf{0.308} & 0.55  & 0.348 & 0.442 & 0.586\\
4B(HACPO)    & \textbf{0.808} & \textbf{0.801} & 0.903 & 0.267 & \textbf{0.575} & \textbf{0.386} & \textbf{0.467} & \textbf{0.601}\\
\bottomrule
\end{tabular}
\end{center}
\vskip -0.1in
\end{table*}

\textbf{Model Capabilities Discrepancy  Coefficient.}
We isolate this coefficient in gradient modulation by disabling it in Eq.\ref{eq:capability_scaled_adv}, while retaining it for advantage estimation. Table \ref{Table: Ablation of Model Capabilities Discrepancy Coefficient} confirms that removing this modulation degrades performance. This validates the coefficient's critical function as a capability-aware scaler: it amplifies gradients from stronger agents to accelerate learning, while attenuating updates from weaker ones to mitigate potential noise. 
\begin{table*}[h!]
\scriptsize
\caption{Ablation of Model Capabilities Discrepancy  Coefficient}
\vspace{-0.2cm}
\label{Table: Ablation of Model Capabilities Discrepancy  Coefficient}
\begin{center}
\begin{tabular}{lcccccccccr}
\toprule
Model & MATH-500 & math & gsm8k & aime2025 & ACM23 & minerva & olympiad & AVG\\
\midrule
1.7B(HACPO - $\omega$ )     & 0.666 & 0.657 & 0.806 & 0.105 & 0.425 & 0.25 & \textbf{0.324} & 0.462\\
1.7B(HACPO)  & \textbf{0.69}  & \textbf{0.674} & \textbf{0.822} & \textbf{0.225} & \textbf{0.45}  & \textbf{0.279} & 0.314 & \textbf{0.493}\\
4B(HACPO - $\omega$)      & 0.803 & 0.797 & 0.902 & 0.261 & 0.55  & \textbf{0.401} & \textbf{0.475} & 0.6 \\
4B(HACPO)    & \textbf{0.808} & \textbf{0.801} & \textbf{0.903} & \textbf{0.267} & \textbf{0.575} & 0.386 & 0.467 & \textbf{0.601}\\
\bottomrule
\end{tabular}
\end{center}
\vskip -0.15in
\end{table*}

\textbf{Exponential Importance Sampling.}
We examined the impact of $\alpha$ on Qwen3-1.7B/4B-Base and Qwen3-4B/8B-Base combinations (Table \ref{Table: Ablation of Exponential Importance Sampling}). Results highlight a critical trade-off: increasing $\alpha$ enforces a more conservative policy towards cross-agent responses, which aids stability by suppressing large distribution shifts but hinders efficiency by reducing the effective learning signal. Thus, the optimal $\alpha$ is model combination dependent, necessitating a balance between stable convergence and maximal information extraction.

\textbf{Stepwise Clipping.}
We assess the necessity of this mechanism on the Qwen3-4B/8B-Base combination. As visualized in Figure \ref{fig:The Ablation of Stepwise Clipping}, removing the clipping constraint (\textbf{no Clip}) causes severe instability, while omitting the stepwise schedule (\textbf{no Stepwise}) leads to suboptimal convergence compared to the full HACPO. This confirms that the stepwise clipping is indispensable for stabilizing collaborative learning, as neither unconstrained nor statically bounded updates suffice to handle high-variance cross-agent responses.

\begin{figure}[h!]
    \centering
    \begin{minipage}[c]{0.42\textwidth}
        \centering
        \footnotesize
        \captionof{table}{The Impact of $\alpha$ (MATH500)}
        \label{Table: Ablation of Exponential Importance Sampling}
        \begin{tabular}{lcccr}
        \toprule
        $\alpha$ & 0.0 & \textbf{1.0} & 2.0 & 3.0 \\
        \midrule
        \multicolumn{5}{c}{Qwen3-1.7B-Base and Qwen3-4B-Base} \\
        \midrule
        1.7B-Base  & 0.63 & 0.664 & 0.654 & \textbf{0.668} \\
        4B-Base & 0.756 & \textbf{0.792} & 0.768 & 0.77 \\
        \midrule
        \multicolumn{5}{c}{Qwen3-4B-Base and Qwen3-8B-Base} \\
        \midrule
        4B-Base  & 0.772 & 0.776 & 0.77 & \textbf{0.776} \\
        8B-Base & 0.764 & 0.772 & 0.766 & \textbf{0.778} \\
        \bottomrule
        \end{tabular}
    \end{minipage}
    \hfill
    \begin{minipage}[c]{0.56\textwidth}
        \centering
        \subfloat[Qwen3-4B-Base]{\includegraphics[width=0.46\linewidth]{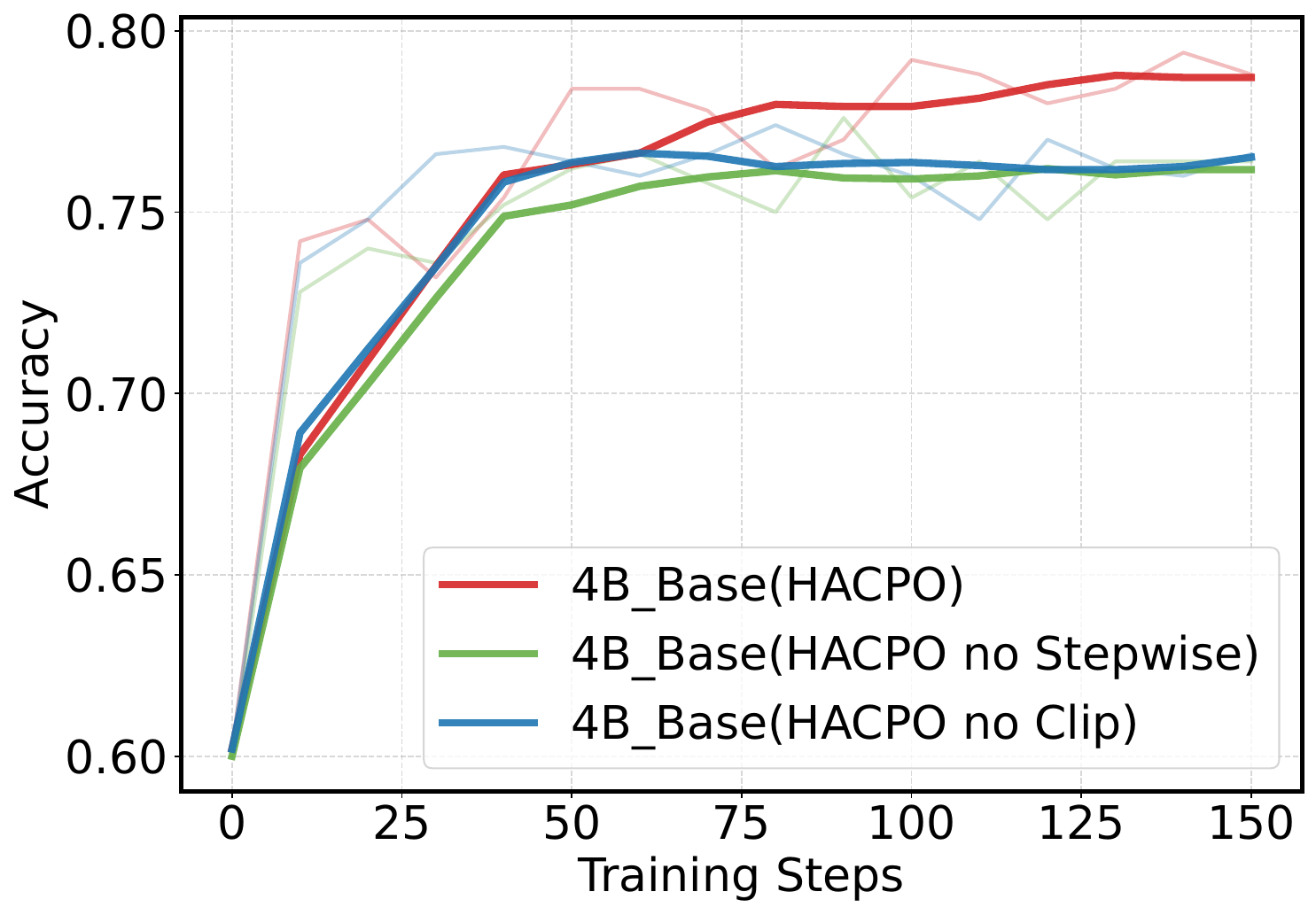}}
        \hspace{0.3cm}%
        \subfloat[Qwen3-8B-Base]{\includegraphics[width=0.46\linewidth]{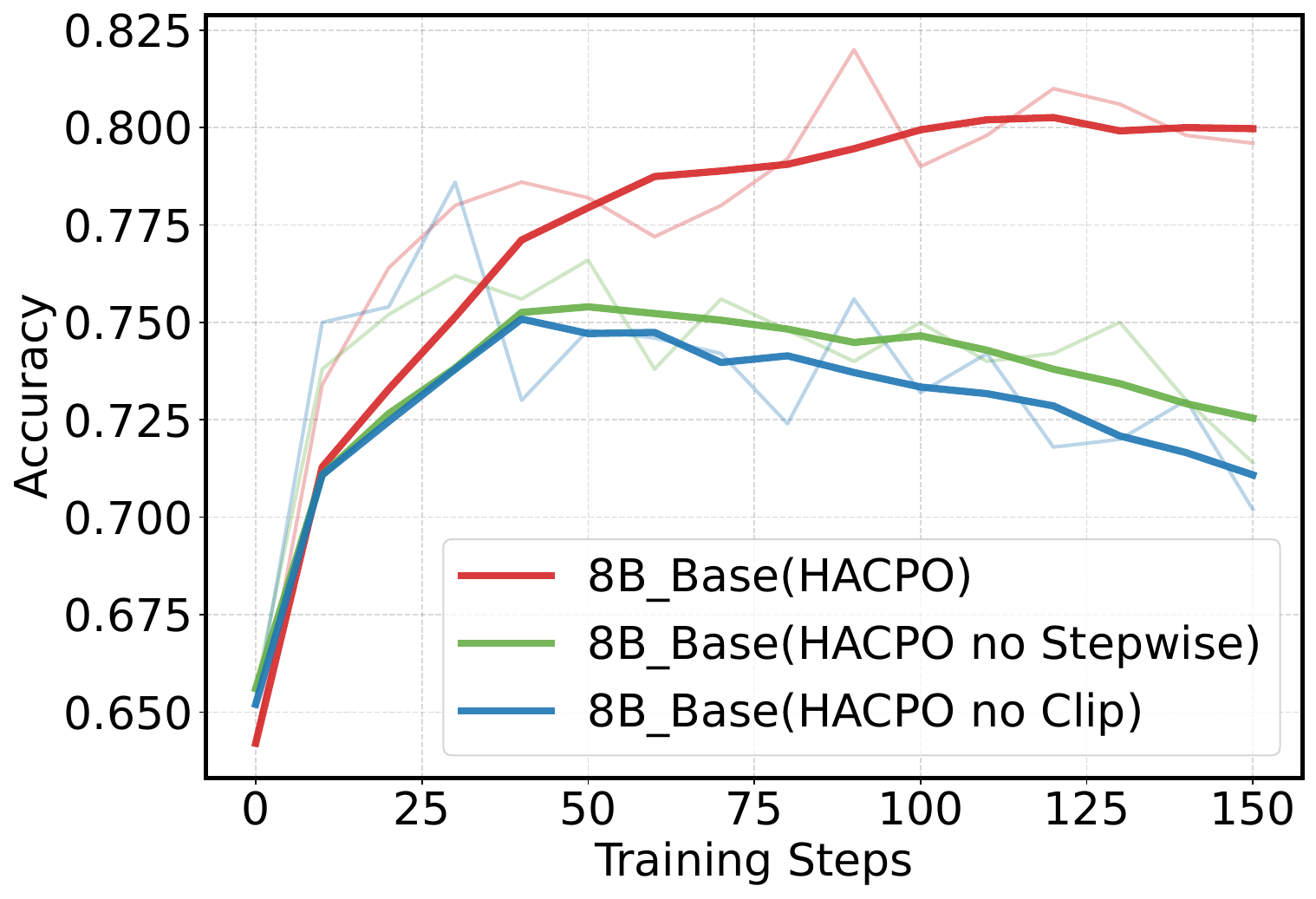}}
        \caption{Ablation of Stepwise Clipping (MATH500)}
        \label{fig:The Ablation of Stepwise Clipping}
    \end{minipage}
    
    \vskip -0.2in
\end{figure}


\section{Related Work}
\label{sec: Related Work}
Our work is most closely related to Reinforcement Learning with Verifiable Rewards (RLVR), with Group Sequence Policy Optimization (GSPO) being the most relevant prior study. GSPO demonstrates the efficacy of sequence-level importance sampling in Mixture-of-Experts (MoE) models, where tokens may originate from different networks. This insight inspires our approach to facilitate rollout sharing among heterogeneous agents. Additionally, our work shares conceptual parallels with Multi-Agent Reinforcement Learning (MARL). A more detailed discussion of related work is provided in Appendix \ref{sec: related work}.

\section{Conclusion}
\label{sec: Conclusion}
We propose HACRL, a collaborative multi-agent reinforcement learning problem tailored for heterogeneous agent ecosystems.
HACRL enables principled rollout sharing among heterogeneous agents, improving sample utilization efficiency while promoting cross-agent knowledge transfer.
To instantiate this problem, we introduce HACPO, which incorporates four tailored mechanisms to mitigate capability discrepancies and policy distribution shifts arising during collaborative policy optimization.
Our theoretical analysis establishes the unbiasedness of the oracle capability-aware baseline and bounds the finite-batch capability ratio estimation error under the stated assumptions.
Extensive experiments demonstrate that HACPO consistently and significantly improves performance across all heterogeneity types.

\section{Limitation and Future Work}
\label{sec: Limitation and Future Work}
While HACPO inherently supports $n \ge 3$ agents, current empirical evaluations are restricted to two- and three-agent settings due to prohibitive computational resource and non-trivial infrastructure modifications required within the verl framework. Nevertheless, our extensive validation across six diverse two-agent model combinations and three distinct types of heterogeneity serves as a strong empirical proxy. The consistent, robust gains observed across these highly varied two-agent pairs and the three-agent setting indicate that HACPO's collaborative mechanisms will generalize effectively to larger agent ecosystems. Future work will address these engineering bottlenecks to empirically scale HACPO.

\clearpage
\begin{ack}

This research was supported by China NSFC (No. 62506024). The authors declare no competing interests.
\end{ack}

\small
\bibliography{main}
\bibliographystyle{plainnat}
\medskip
\clearpage
\normalsize

\appendix

\section{Training and Evaluation Details}
\label{sec: Training and Evalution Details}
All experiments in this paper are conducted using verl \cite{verl}. In the experiments, we set the maximum prompt length to 1024 and the maximum response length to 4096. We use the MATH dataset for training. The learning rate is set to $1 \times 10^{-6}$. For the responses generated by the trained agents in HACPO or single GSPO, we set $\epsilon_{\text{low}} = 0.0003$ and $\epsilon_{\text{high}} = 0.0004$, which is consistent with the setting mentioned in GSPO\cite{GSPO}. As for the single GRPO, we set $\epsilon_{\text{low}} = 0.2$ and $\epsilon_{\text{high}} = 0.28$, which follows the trick mentioned in DAPO \cite{yu2025dapo} and is widely used. The batch size is set to 128, with a mini-batch size of 64 and $n=8$ rollouts per prompt. In the Resource-Equivalent Baseline (GSPO$\times$2), we use a mini-batch size of 32 and $n=16$ rollouts per prompt to ensure double updates per step, while maintaining a consistent number of rollouts per update with other settings. We train for one epoch, except when examining the impact of stepwise clipping on stabilizing the training process. During evaluation, due to the high complexity of benchmarks such as AIME2025, we adopt a maximum response length of 8192 tokens in the main experiments and the ablation of Agent-Capability-Aware Advantage Estimator and Model Capability Discrepancy Coefficient (Table \ref{Main Results}, Table \ref{Additional Experimental Results}, Table \ref{Table: Ablation of Model Capabilities Discrepancy  Coefficient} and Table \ref{Ablation of Advantage Estimator}). The three-agent and distillation comparisons, as well as the reevaluated seed experiments (Table~\ref{tab: seed experiments}), also use the 8192-token evaluation limit. For all other experiments, the maximum response length is kept consistent with the training configuration and is set to 4096 tokens. Experiments are conducted on eight NVIDIA A800 GPUs, each with 80\,GB of memory.

Regarding the models used in our experiments, We employed the Qwen3 \cite{yang2025qwen3} and Llama3.2 \cite{llama} series of models, as well as Gemma3-1B-Instruct \cite{gemma2025gemma3} in the three-agent experiment. In detail, Qwen3-(1.7B/4B/8B)-Base denotes the base models, while Qwen3-(1.7B/4B/8B) refer to the distilled variants obtained through strong model distillation from their corresponding base models. In addition, Qwen3-4B-Instruct is a further fine-tuned version of Qwen3-4B, designed to better follow user instructions and generate more accurate responses.

For the clipping boundary $\delta$ in the exponential importance sampling of $\alpha$, as well as the gradient clipping step size $\delta_{\text{step}}$, each experiment has slight variations. We provide the specific settings used for each experiment in the Table \ref{The Details of Hyperparameter}. A commonly used set of parameters is $\alpha = 1$, 1-$\delta = 0.8$, and $\delta_{\text{step}} = 0.025$.

\begin{table}[h!]
\caption{The Details of Hyperparameter}
\label{The Details of Hyperparameter}
\vskip 0.05in
\begin{center}
\begin{small}
\begin{tabular}{lccccccr}
\toprule
Model Combination & $\alpha$ & 1-$\delta$ & $\delta_{step}$ \\
\midrule
Qwen3-4B and Qwen3-4B-Instruct & 3.0 & 0.8 & 0.01 \\
Qwen3-1.7B-Base and Qwen3-4B-Base & 1.0 & 0.8 & 0.025 \\
Qwen3-4B-Base and Qwen3-8B-Base   & 3.0 & 0.8 & 0.025 \\
Llama3.2-1B-Instruct and Llama3.2-3B-Instruct & 1.0 & 0.9 & 0.01 \\
Qwen3-1.7B-Base and Llama3.2-1B-Instruct & 1.0 & 0.8 & 0.025\\
Qwen3-4B-Base and Llama3.2-3B-Instruct & 1.0 & 0.8 & 0.025 \\
\bottomrule
\end{tabular}
\end{small}
\end{center}
\vskip -0.1in
\end{table}

\section{Additional Related Work}
\label{sec: related work}

\subsection{Reinforcement Learning with Verifiable Rewards}
Group Relative Policy Optimization (GRPO)~\cite{shao2024deepseekmath} estimates advantages from groups of responses to the same prompt, avoiding a separate critic network. Subsequent work studies group-based advantage estimation~\cite{yang2026grouprelativeadvantagebiased} and develops variants that improve optimization~\cite{yu2025dapo,dcpo,gmpo,wang2025aspo,liu2025graph,wang2026policy,kong2026workflowr1,lu2026contextual,li2026adaptive}. Group Sequence Policy Optimization (GSPO)~\cite{GSPO} uses sequence-level importance sampling to improve training stability, including for Mixture-of-Experts models. HACPO builds on this sequence-level objective and addresses the additional capability and policy-distribution differences that arise when heterogeneous agents directly reuse one another's rollouts.

Related post-training work studies adaptive sample scheduling for DPO~\cite{huang2025adaptive}, guided rollout sampling for efficient reasoning~\cite{huang2026does}, and reward-model adaptation to changing policy distributions~\cite{huang2026real}. In adjacent work on model-dependent supervision, Model-Aligned Coupling (MAC) selects source--target couplings for flow matching according to the current model's prediction error~\cite{lin2026mac}. Beyond mathematical reasoning, DEFT-RLVR explores verifiable reasoning in autonomous-driving VLMs by deferring exposure to candidate future trajectories~\cite{huang2026deferred}.

\subsection{Collaborative Learning and Scope of Novelty}
\label{sec:positioning}
Training-time collaboration followed by independent deployment is an established idea, including in LLM co-training methods such as CORY~\cite{cory}, WMSS~\cite{chen2026weak} and ILR~\cite{ILR}. HACRL studies a specific problem within this broader setting: bidirectional sharing of verified rollouts among heterogeneous LLM agents in RLVR. Our contribution is to formulate this problem across state, size, and model heterogeneity and develop HACPO to address the resulting capability differences and cross-policy distribution shifts.

CORY~\cite{cory} initializes a pioneer and an observer as two copies of the same LLM. The pioneer responds to the query, and the observer conditions on both the query and the pioneer's response. Periodic role exchange allows either model to be used independently after training. ILR~\cite{ILR} also targets independent reasoning after collaborative training, but transfers information through explicit inter-agent discussions and calibrates rewards using the other model's reward distribution. HACRL instead lets each agent generate responses independently and directly reuses complete verified rollouts for policy updates, without conditioning generation on another agent's response.

SAPO implements decentralized, asynchronous rollout sharing across distributed training nodes and validates it in controlled experiments with eight homogeneous Qwen2.5-0.5B agents~\cite{amico2025sharing}.

Other methods optimize agents for a collaborative inference system\cite{zou2026recursivemas,zou2025transformer,zou2026latent}. MAPoRL~\cite{MAPoRL} co-trains LLMs for multi-turn discussion to improve their joint answer, while MALT~\cite{MALT} trains specialized generator, verifier, and refiner agents for a sequential reasoning pipeline. In HACRL, each trained model solves a request without communicating with its training partners or relying on specialized inference-time roles.

\subsection{Knowledge Distillation}
Knowledge distillation transfers information between models~\cite{hinton2015distilling,distillation}; its variants include online collaboration~\cite{anil2018large} and on-policy training using student-generated responses~\cite{on-policy_distillation,qu2026spot}. Our experimental comparison concerns specific one-way baselines: OPD and OPD+GRPO update the weaker model using a frozen stronger teacher, whereas HACPO updates every participating model using shared verified rollouts. We also evaluate HACPO with a frozen stronger partner to isolate the benefit of joint optimization (Appendix~\ref{sec:distillation_baselines}). A separate sequential train-then-distill comparison first trains the stronger model and then uses its frozen checkpoint to train the weaker model (Appendix~\ref{sec:train_then_distill}). Together, these comparisons assess whether bidirectional collaboration improves the target model portfolio beyond the evaluated one-way alternatives.

Teacher-assisted post-training also improves the available learning signals. MIND combines reinforcement learning with supervised updates on locally corrected student responses~\cite{liu2026automated}, while ToolAnchor verifies interventions proposed by diverse teachers through student rollouts before incorporating them into agentic post-training~\cite{liu2026toolanchor}.

\section{Heterogeneous Agent Importance Sampling Analysis}
\label{Heterogeneous Agent Importance Sampling Analysis}
In the reinforcement learning paradigm, importance sampling is commonly used to stabilize updates, often through a clipping mechanism. The clipping range typically centers around 1.0. For instance, in GSPO, the upper and lower bounds for clipping are set to 1.0004 and 0.9997, respectively. However, in a multi-agent setting, the importance sampling values for samples from other agents do not exhibit the same pattern and fluctuate as training progresses.

In the experiment involving Qwen3-1.7B-Base and Qwen3-4B-Base, we distinguish between self-generated responses and cross-agent responses, denoted as \(s_{t,i}^{(1,1)}\) and \(s_{t,i}^{(1,2)}\), respectively. These values represent the average importance sampling across each training step. It is important to note that while \(s_{t,i}^{(1,1)}\) remains stable and tends to stay around 1 throughout training, \(s_{t,i}^{(1,2)}\) does not follow a fixed range and fluctuates as training progresses. The results are shown in Table \ref{IS in all steps}

\begin{table}[h!]
\caption {\(s_{t,i}^{(1,1)}\) and \(s_{t,i}^{(1,2)}\) of Qwen3-1.7B-Base in all steps}
\label{IS in all steps}
\vskip 0.05in
\begin{center}
\begin{small}
\begin{tabular}{lccccccr}
\toprule
Model & mean & max & min & range\\
\midrule
$s^{homo}$ & 1.00002 & 1.00020 & 0.99960 & 0.00060 \\
$s^{hete}$ & 0.89550 & 0.93615 & 0.86198 & 0.07417 \\
\bottomrule
\end{tabular}
\end{small}
\end{center}
\vskip -0.1in
\end{table}

For self-generated responses, as the number of updates(mini batches) within a batch increases, the discrepancy between the sampling policy \( \pi_{\theta_\text{old}} \) and the current policy \( \pi_\theta \) grows, leading to an increased \(s_{t,i}^{(1,1)}\) and a higher ratio of clipped tokens. However, for cross-agent responses, the discrepancy between the current policy \( \pi^{(k)}_\theta \) and the sampling model's policy \( \pi_{\theta_\text{old}}^{(j)} \) fluctuates unpredictably, leading to a variable \(s_{t,i}^{(1,2)}\) and the ratio of clipped tokens. 

In a batch with multiple mini-batches, as the number of updates increases, self-generated responses become more heavily clipped in later mini-batches due to the growing discrepancy between the current and old policies. Therefore, the influence of cross-agent responses is likely to increase in later mini-batches, as their importance sampling values are less predictable, leading to an instability if they dominate the update.

\section{Theoretical Analysis}
\label{sec: Theoretical Analysis}
To ensure mathematical rigor, the appendix proofs use the expanded notation $R(x, y)$ for rewards and $\pi_t$ for policies at step $t$. All other symbols are consistent with the main text.
\subsection{Oracle Capability Baseline}
\label{sec:oracle_capability_ratio_concentration}

This subsection defines the oracle capability ratio using expectations over the
prompt distribution and the response distributions of the agents. It proves that
the capability-aware baseline in Eq.~\eqref{eq:capability_aware_adv} has the
same expectation as the reward of the learner agent. All expectations below are
taken jointly over prompt sampling and response sampling.

We first define the expected reward of each agent and the batch-level reward mean used by the oracle baseline.

\begin{assumption}[Finite expected reward]
\label{ass:global_reward_model}
Fix a training step $t$ and $n$ agents. Let $\mathcal D_t$ be the prompt
distribution at step $t$. For each agent $a\in\{1,\dots,n\}$, let
$\pi_t^{(a)}(\cdot\mid x)$ denote the response distribution of agent $a$ given
prompt $x$. Assume the reward has finite expectation for each agent:
\begin{equation}
\label{eq:finite_global_reward}
\mathbb E_{x\sim\mathcal D_t,\; y\sim\pi_t^{(a)}(\cdot\mid x)}
\left[|R(x,y)|\right]
<
\infty .
\end{equation}
Define the conditional prompt value and the prompt-averaged expected reward as
\begin{align}
\label{eq:conditional_prompt_value}
q_t^{(a)}(x)
&:=
\mathbb E_{y\sim\pi_t^{(a)}(\cdot\mid x)}
\left[R(x,y)\right], \\
\label{eq:global_expected_reward}
p_t^{(a)}
&:=
\mathbb E_{x\sim\mathcal D_t}
\left[q_t^{(a)}(x)\right]
=
\mathbb E_{x\sim\mathcal D_t,\; y\sim\pi_t^{(a)}(\cdot\mid x)}
\left[R(x,y)\right].
\end{align}
\end{assumption}

\begin{assumption}[Positive oracle denominators]
\label{ass:positive_oracle_denominators}
Throughout the oracle baseline construction, for every denominator agent
$j\in\{1,\dots,n\}$,
\begin{equation}
\label{eq:positive_oracle_denominators}
p_t^{(j)}>0 .
\end{equation}
\end{assumption}

\begin{definition}[Oracle capability ratio]
\label{def:oracle_capability_ratio}
For agents $k$ and $j$ with $p_t^{(j)}>0$, the oracle capability ratio is
\begin{equation}
\label{eq:oracle_capability_ratio}
\bar\omega_t^{(k,j)}
:=
\frac{p_t^{(k)}}{p_t^{(j)}}.
\end{equation}
In particular, $\bar\omega_t^{(k,k)}=1$.
\end{definition}

\begin{definition}[Oracle capability-aware baseline]
\label{def:oracle_capability_baseline}
Let the batch at step $t$ contain $B$ prompt instances
$X_{t,1},\dots,X_{t,B}$ sampled independently from $\mathcal D_t$. Conditional
on $X_{t,b}$, let
$Y_{t,b,1}^{(j)},\dots,Y_{t,b,G}^{(j)}$ be sampled independently from
$\pi_t^{(j)}(\cdot\mid X_{t,b})$. Define
\begin{equation}
\label{eq:batch_level_reward_mean}
\widetilde P_t^{(j)}
:=
\frac{1}{BG}
\sum_{b=1}^{B}
\sum_{i=1}^{G}
R\!\left(X_{t,b},Y_{t,b,i}^{(j)}\right)
\end{equation}
as the batch-level empirical mean reward of agent $j$. The oracle
capability-aware baseline for learner agent $k$ is
\begin{equation}
\label{eq:oracle_capability_baseline}
\mu_{t,\star}^{(k)}
:=
\frac{1}{n}
\sum_{j=1}^{n}
\bar\omega_t^{(k,j)}\widetilde P_t^{(j)}.
\end{equation}
\end{definition}

\begin{theorem}[Unbiased oracle capability-aware baseline]
\label{thm:oracle_unbiased_baseline}
Under Assumptions~\ref{ass:global_reward_model} and
\ref{ass:positive_oracle_denominators}, the oracle capability-aware baseline
$\mu_{t,\star}^{(k)}$ in Definition~\ref{def:oracle_capability_baseline}
satisfies
\begin{equation}
\label{eq:oracle_baseline_unbiased}
\mathbb E\!\left[\mu_{t,\star}^{(k)}\right]
=
p_t^{(k)}
=
\mathbb E_{x\sim\mathcal D_t,\; y\sim\pi_t^{(k)}(\cdot\mid x)}
\left[R(x,y)\right].
\end{equation}
\end{theorem}

\begin{proof}
By the definition of $\widetilde P_t^{(j)}$ and
Assumption~\ref{ass:global_reward_model},
\begin{equation}
\label{eq:batch_level_mean_unbiased}
\mathbb E\!\left[\widetilde P_t^{(j)}\right]
=
p_t^{(j)}.
\end{equation}
Therefore,
\begin{align}
\label{eq:oracle_baseline_unbiased_proof}
\mathbb E\!\left[\mu_{t,\star}^{(k)}\right]
&=
\frac{1}{n}
\sum_{j=1}^{n}
\bar\omega_t^{(k,j)}
\mathbb E\!\left[\widetilde P_t^{(j)}\right] \\
&=
\frac{1}{n}
\sum_{j=1}^{n}
\frac{p_t^{(k)}}{p_t^{(j)}}p_t^{(j)}
=
p_t^{(k)}.
\end{align}
\end{proof}

\begin{corollary}[Zero-mean oracle centered reward]
\label{cor:oracle_zero_mean_centered_reward}
Let $y\sim\pi_t^{(k)}(\cdot\mid x)$ with $x\sim\mathcal D_t$, and define the
oracle centered reward
\begin{equation}
\label{eq:oracle_centered_reward}
\bar A_{t,\star}^{(k)}(x,y)
:=
R(x,y)-\mu_{t,\star}^{(k)}.
\end{equation}
Under the conditions of Theorem~\ref{thm:oracle_unbiased_baseline},
\begin{equation}
\label{eq:oracle_centered_reward_zero}
\mathbb E\!\left[\bar A_{t,\star}^{(k)}(x,y)\right]=0.
\end{equation}
\end{corollary}

\begin{proof}
By linearity of expectation,
\begin{equation}
\label{eq:oracle_centered_reward_zero_proof}
\mathbb E\!\left[\bar A_{t,\star}^{(k)}(x,y)\right]
=
\mathbb E_{x\sim\mathcal D_t,\; y\sim\pi_t^{(k)}(\cdot\mid x)}
\left[R(x,y)\right]
-
\mathbb E\!\left[\mu_{t,\star}^{(k)}\right].
\end{equation}
Theorem~\ref{thm:oracle_unbiased_baseline} shows that the two terms are equal,
so the difference is zero.
\end{proof}

\begin{remark}[Single-agent degeneration]
\label{rem:single_agent_degeneration}
When $n=1$, the oracle ratio and baseline become
\begin{equation}
\label{eq:single_agent_oracle_degenerate}
\bar\omega_t^{(k,k)}=1,
\qquad
\mu_{t,\star}^{(k)}=\widetilde P_t^{(k)}.
\end{equation}
Thus the multi-agent statement reduces to the standard single-agent
sample-mean baseline.
\end{remark}

\subsection{Empirical Ratio Concentration}
\label{sec: Empirical Ratio Concentration}
The next result bounds the deviation between the batch-level empirical ratio
and the oracle ratio.

\begin{assumption}[Bounded batch-level sampling]
\label{ass:bounded_batch_level_sampling}
In addition to Assumption~\ref{ass:global_reward_model}, assume the reward is
bounded:
\begin{equation}
\label{eq:bounded_reward}
0\le R(x,y)\le 1.
\end{equation}
Let $B$ be the number of prompts in the batch at step $t$, and let $G$ be the
number of responses sampled for each prompt and agent. Define
\begin{equation}
\label{eq:batch_level_sample_size}
N_{\mathrm{batch}}:=BG .
\end{equation}
The batch at step $t$ is sampled under the prompt distribution and agent
policies in Assumption~\ref{ass:global_reward_model}. For each prompt instance
$b\in\{1,\dots,B\}$, let $X_b\sim\mathcal D_t$ be sampled independently.
Conditional on $X_b$, responses
$Y_{b,1}^{(a)},\dots,Y_{b,G}^{(a)}$ are sampled independently from
$\pi_t^{(a)}(\cdot\mid X_b)$ for each agent $a$, and different prompt groups are
independent. The observed reward samples are
\begin{equation}
\label{eq:batch_reward_sample}
Z_{b,g}^{(a)}
:=
R\!\left(X_b,Y_{b,g}^{(a)}\right),
\qquad
b=1,\dots,B,\quad g=1,\dots,G .
\end{equation}
\end{assumption}

\begin{definition}[Batch-level empirical capability ratio]
\label{def:empirical_capability_ratio}
Using the batch-level samples in Eq.~\eqref{eq:batch_reward_sample}, define
\begin{equation}
\label{eq:empirical_capability_mean_ratio}
\hat P_t^{(a)}
:=
\frac{1}{N_{\mathrm{batch}}}
\sum_{b=1}^{B}
\sum_{g=1}^{G}
Z_{b,g}^{(a)},
\qquad
\hat\omega_t^{(k,j)}
:=
\frac{\hat P_t^{(k)}}{\hat P_t^{(j)}}.
\end{equation}
\end{definition}

\begin{lemma}[Batch-level concentration of empirical capabilities]
\label{lemma:capability_concentration}
Under Assumption~\ref{ass:bounded_batch_level_sampling}, for any
$\delta\in(0,1)$, with probability at least $1-\delta$, all agents
simultaneously satisfy
\begin{equation}
\label{eq:uniform_capability_concentration}
\left|\hat P_t^{(a)}-p_t^{(a)}\right|
\le
\epsilon_{\mathrm{batch}}(\delta),
\qquad
\text{for all } a=1,\dots,n,
\end{equation}
where
\begin{equation}
\label{eq:batch_level_epsilon}
\epsilon_{\mathrm{batch}}(\delta)
:=
\epsilon_{\mathrm{prompt}}(\delta)
+
\epsilon_{\mathrm{resp}}(\delta),
\end{equation}
with
\begin{align}
\label{eq:prompt_concentration_radius}
\epsilon_{\mathrm{prompt}}(\delta)
&:=
\sqrt{\frac{\log(4n/\delta)}{2B}}, \\
\label{eq:response_concentration_radius}
\epsilon_{\mathrm{resp}}(\delta)
&:=
\sqrt{\frac{\log(4n/\delta)}{2BG}} .
\end{align}
\end{lemma}

\begin{proof}
For a fixed agent $a$, define the empirical prompt-value average
\begin{equation}
\label{eq:empirical_prompt_value_average}
\bar q_{t,\mathrm{batch}}^{(a)}
:=
\frac{1}{B}
\sum_{b=1}^{B}
q_t^{(a)}(X_b).
\end{equation}
Then
\begin{equation}
\label{eq:batch_level_error_decomposition}
\hat P_t^{(a)}-p_t^{(a)}
=
\left(
\hat P_t^{(a)}-\bar q_{t,\mathrm{batch}}^{(a)}
\right)
+
\left(
\bar q_{t,\mathrm{batch}}^{(a)}-p_t^{(a)}
\right).
\end{equation}
Since $q_t^{(a)}(X_b)\in[0,1]$ and the prompt instances are independent,
Hoeffding's inequality gives
\begin{equation}
\label{eq:prompt_hoeffding}
\Pr\!\left(
\left|
\bar q_{t,\mathrm{batch}}^{(a)}-p_t^{(a)}
\right|
\ge
\epsilon
\right)
\le
2\exp(-2B\epsilon^2).
\end{equation}
Conditional on the prompt instances, the rewards $Z_{b,g}^{(a)}$ are
independent and bounded in $[0,1]$, with conditional means
$q_t^{(a)}(X_b)$. Therefore,
\begin{equation}
\label{eq:response_hoeffding_conditional}
\Pr\!\left(
\left|
\hat P_t^{(a)}-\bar q_{t,\mathrm{batch}}^{(a)}
\right|
\ge
\epsilon
\;\middle|\;
X_1,\dots,X_B
\right)
\le
2\exp(-2BG\epsilon^2).
\end{equation}
Taking expectation over the prompt instances gives the unconditional bound
\begin{equation}
\label{eq:response_hoeffding_unconditional}
\Pr\!\left(
\left|
\hat P_t^{(a)}-\bar q_{t,\mathrm{batch}}^{(a)}
\right|
\ge
\epsilon
\right)
\le
2\exp(-2BG\epsilon^2).
\end{equation}
Choose $\epsilon_{\mathrm{prompt}}(\delta)$ and
$\epsilon_{\mathrm{resp}}(\delta)$ as in
Eqs.~\eqref{eq:prompt_concentration_radius}--\eqref{eq:response_concentration_radius}.
For each agent, the prompt and response failure probabilities are each at most
$\delta/(2n)$. Applying the union bound over both terms and over the $n$ agents
gives Eq.~\eqref{eq:uniform_capability_concentration}.
\end{proof}

\begin{theorem}[High-probability bound for the empirical capability ratio]
\label{thm:empirical_ratio_bound}
Suppose Assumption~\ref{ass:bounded_batch_level_sampling} holds. Let
$\epsilon_{\mathrm{batch}}=\epsilon_{\mathrm{batch}}(\delta)$ be defined as in
Eq.~\eqref{eq:batch_level_epsilon}, and assume
$\epsilon_{\mathrm{batch}}< p_t^{(j)}$. Then with probability at least $1-\delta$,
\begin{equation}
\label{eq:capability_ratio_bound}
\frac{p_t^{(k)}-\epsilon_{\mathrm{batch}}}{p_t^{(j)}+\epsilon_{\mathrm{batch}}}
\le
\hat\omega_t^{(k,j)}
\le
\frac{p_t^{(k)}+\epsilon_{\mathrm{batch}}}{p_t^{(j)}-\epsilon_{\mathrm{batch}}}.
\end{equation}
Moreover,
\begin{equation}
\label{eq:capability_ratio_deviation_bound}
\left|
\hat\omega_t^{(k,j)}
-
\bar\omega_t^{(k,j)}
\right|
\le
\frac{\epsilon_{\mathrm{batch}}\left(p_t^{(j)}+p_t^{(k)}\right)}
{p_t^{(j)}\left(p_t^{(j)}-\epsilon_{\mathrm{batch}}\right)}.
\end{equation}
If, in addition, there exists a constant $\gamma>0$ such that
$p_t^{(j)}\ge\gamma$ and $\epsilon_{\mathrm{batch}}\le \gamma/2$, then
\begin{equation}
\label{eq:capability_ratio_gamma_bound}
\left|
\hat\omega_t^{(k,j)}
-
\bar\omega_t^{(k,j)}
\right|
\le
\frac{4}{\gamma^2}
\left(
\sqrt{\frac{\log(4n/\delta)}{2B}}
+
\sqrt{\frac{\log(4n/\delta)}{2BG}}
\right).
\end{equation}
\end{theorem}

\begin{proof}
On the event from Lemma~\ref{lemma:capability_concentration}, we have
\begin{equation}
\label{eq:capability_mean_confidence_interval}
p_t^{(a)}-\epsilon_{\mathrm{batch}}
\le
\hat P_t^{(a)}
\le
p_t^{(a)}+\epsilon_{\mathrm{batch}}
\end{equation}
for every agent $a$. Since $\epsilon_{\mathrm{batch}}<p_t^{(j)}$, the
denominator $\hat P_t^{(j)}$ is positive. Therefore,
\begin{equation}
\label{eq:capability_ratio_bound_proof}
\frac{p_t^{(k)}-\epsilon_{\mathrm{batch}}}{p_t^{(j)}+\epsilon_{\mathrm{batch}}}
\le
\frac{\hat P_t^{(k)}}{\hat P_t^{(j)}}
\le
\frac{p_t^{(k)}+\epsilon_{\mathrm{batch}}}{p_t^{(j)}-\epsilon_{\mathrm{batch}}}.
\end{equation}
This proves Eq.~\eqref{eq:capability_ratio_bound}.

For the deviation bound, write
\begin{align}
\label{eq:ratio_deviation_algebra}
\left|
\frac{\hat P_t^{(k)}}{\hat P_t^{(j)}}
-
\frac{p_t^{(k)}}{p_t^{(j)}}
\right|
&=
\left|
\frac{\hat P_t^{(k)}p_t^{(j)}
-p_t^{(k)}\hat P_t^{(j)}}
{\hat P_t^{(j)}p_t^{(j)}}
\right| \nonumber \\
&\le
\frac{
p_t^{(j)}\left|\hat P_t^{(k)}-p_t^{(k)}\right|
+p_t^{(k)}\left|\hat P_t^{(j)}-p_t^{(j)}\right|
}
{\hat P_t^{(j)}p_t^{(j)}} \nonumber \\
&\le
\frac{\epsilon_{\mathrm{batch}}\left(p_t^{(j)}+p_t^{(k)}\right)}
{p_t^{(j)}\left(p_t^{(j)}-\epsilon_{\mathrm{batch}}\right)}.
\end{align}
If $p_t^{(j)}\ge\gamma$ and $\epsilon_{\mathrm{batch}}\le\gamma/2$, then
\begin{equation}
\label{eq:denominator_gamma_lower_bound}
p_t^{(j)}-\epsilon_{\mathrm{batch}}\ge\gamma/2.
\end{equation}
Since rewards lie in $[0,1]$, we also have
\begin{equation}
\label{eq:bounded_expected_reward_sum}
p_t^{(j)}+p_t^{(k)}\le 2.
\end{equation}
Combining Eqs.~\eqref{eq:ratio_deviation_algebra}--\eqref{eq:bounded_expected_reward_sum}
with Eq.~\eqref{eq:batch_level_epsilon} proves Eq.~\eqref{eq:capability_ratio_gamma_bound}.
\end{proof}

\begin{definition}[Empirical capability-aware baseline]
\label{def:empirical_global_baseline}
Replacing the oracle ratio in Definition~\ref{def:oracle_capability_baseline}
with the batch-level empirical ratio gives
\begin{equation}
\label{eq:empirical_global_baseline}
\hat\mu_{t,\mathrm{emp}}^{(k)}
:=
\frac{1}{n}
\sum_{j=1}^{n}
\hat\omega_t^{(k,j)}\widetilde P_t^{(j)}.
\end{equation}
\end{definition}

\begin{corollary}[Baseline error induced by empirical ratios]
\label{cor:empirical_ratio_baseline_error}
Under the event in Lemma~\ref{lemma:capability_concentration}, suppose
$\epsilon_{\mathrm{batch}}<p_t^{(j)}$ for every denominator agent $j$. Then
\begin{equation}
\label{eq:empirical_baseline_error_bound}
\left|
\hat\mu_{t,\mathrm{emp}}^{(k)}
-
\mu_{t,\star}^{(k)}
\right|
\le
\frac{1}{n}
\sum_{j=1}^{n}
\widetilde P_t^{(j)}
\frac{\epsilon_{\mathrm{batch}}\left(p_t^{(j)}+p_t^{(k)}\right)}
{p_t^{(j)}\left(p_t^{(j)}-\epsilon_{\mathrm{batch}}\right)}.
\end{equation}
Since $\widetilde P_t^{(j)}\in[0,1]$, the same bound also holds after dropping
the multiplicative factor $\widetilde P_t^{(j)}$ from each summand.
\end{corollary}

\begin{proof}
By Definitions~\ref{def:oracle_capability_baseline} and
\ref{def:empirical_global_baseline},
\begin{align}
\label{eq:empirical_baseline_error_proof}
\left|
\hat\mu_{t,\mathrm{emp}}^{(k)}
-
\mu_{t,\star}^{(k)}
\right|
&=
\left|
\frac{1}{n}
\sum_{j=1}^{n}
\left(\hat\omega_t^{(k,j)}-\bar\omega_t^{(k,j)}\right)
\widetilde P_t^{(j)}
\right| \nonumber \\
&\le
\frac{1}{n}
\sum_{j=1}^{n}
\widetilde P_t^{(j)}
\left|
\hat\omega_t^{(k,j)}-\bar\omega_t^{(k,j)}
\right|.
\end{align}
Applying Theorem~\ref{thm:empirical_ratio_bound} to each ratio gives
Eq.~\eqref{eq:empirical_baseline_error_bound}.
\end{proof}

\begin{corollary}[Combined empirical baseline error]
\label{cor:combined_empirical_baseline_error}
For any $\delta\in(0,1)$, define
\begin{equation}
\label{eq:delta_split_batch_radius}
\epsilon_{\mathrm{batch}}^\delta
:=
\epsilon_{\mathrm{batch}}(\delta/2).
\end{equation}
Under Assumption~\ref{ass:bounded_batch_level_sampling}, suppose there exists
$\gamma>0$ such that $p_t^{(j)}\ge\gamma$ for every denominator agent $j$. If
$\epsilon_{\mathrm{batch}}^\delta\le \gamma/2$, then with probability at least
$1-\delta$,
\begin{equation}
\label{eq:combined_empirical_baseline_error}
\left|
\hat\mu_{t,\mathrm{emp}}^{(k)}
-
p_t^{(k)}
\right|
\le
\left(
\frac{4}{\gamma^2}
+
\frac{1}{\gamma}
\right)
\epsilon_{\mathrm{batch}}^\delta .
\end{equation}
\end{corollary}

\begin{proof}
Apply Lemma~\ref{lemma:capability_concentration} with failure probability
$\delta/2$ to the empirical ratio. By
Theorem~\ref{thm:empirical_ratio_bound}, the ratio error then satisfies
\begin{equation}
\label{eq:ratio_error_delta_split}
\left|
\hat\omega_t^{(k,j)}
-
\bar\omega_t^{(k,j)}
\right|
\le
\frac{4}{\gamma^2}\epsilon_{\mathrm{batch}}^\delta
\end{equation}
simultaneously for all denominator agents $j$. Applying the same batch-level
Hoeffding argument to the empirical means in the baseline gives, with failure
probability $\delta/2$,
\begin{equation}
\label{eq:batch_level_mean_concentration}
\left|
\widetilde P_t^{(j)}-p_t^{(j)}
\right|
\le
\epsilon_{\mathrm{batch}}^\delta
\end{equation}
simultaneously for all $j=1,\dots,n$. By the union bound, the two events hold
together with probability at least $1-\delta$.

On this joint event,
\begin{align}
\label{eq:combined_error_decomposition}
\left|
\hat\mu_{t,\mathrm{emp}}^{(k)}
-
p_t^{(k)}
\right|
&\le
\left|
\hat\mu_{t,\mathrm{emp}}^{(k)}
-
\mu_{t,\star}^{(k)}
\right|
+
\left|
\mu_{t,\star}^{(k)}
-
p_t^{(k)}
\right| \nonumber \\
&\le
\frac{4}{\gamma^2}\epsilon_{\mathrm{batch}}^\delta
+
\left|
\frac{1}{n}
\sum_{j=1}^{n}
\bar\omega_t^{(k,j)}
\left(
\widetilde P_t^{(j)}-p_t^{(j)}
\right)
\right| \nonumber \\
&\le
\frac{4}{\gamma^2}\epsilon_{\mathrm{batch}}^\delta
+
\frac{1}{n}
\sum_{j=1}^{n}
\frac{p_t^{(k)}}{p_t^{(j)}}
\epsilon_{\mathrm{batch}}^\delta .
\end{align}
Since rewards are bounded in $[0,1]$, $p_t^{(k)}\le 1$, and
$p_t^{(j)}\ge\gamma$ by the denominator lower-bound condition. Hence
\begin{equation}
\label{eq:batch_level_error_gamma_bound}
\frac{1}{n}
\sum_{j=1}^{n}
\frac{p_t^{(k)}}{p_t^{(j)}}
\epsilon_{\mathrm{batch}}^\delta
\le
\frac{1}{\gamma}\epsilon_{\mathrm{batch}}^\delta .
\end{equation}
Combining Eqs.~\eqref{eq:combined_error_decomposition} and
\eqref{eq:batch_level_error_gamma_bound} proves
Eq.~\eqref{eq:combined_empirical_baseline_error}.
\end{proof}

\begin{remark}[Role of the concentration bound]
\label{rem:role_of_concentration_bound}
Lemma~\ref{lemma:capability_concentration} gives finite-batch control of the
empirical capability means used to form $\hat\omega_t^{(k,j)}$ at a fixed
training step. The two terms in Eq.~\eqref{eq:batch_level_epsilon} correspond to
the two sampling sources in a batch: the $B$ prompt instances determine the
prompt-distribution term $\epsilon_{\mathrm{prompt}}(\delta)$, while the $BG$
response samples determine the response-sampling term
$\epsilon_{\mathrm{resp}}(\delta)$. Theorem~\ref{thm:empirical_ratio_bound}
propagates this mean-reward error through the capability ratio under a positive
denominator condition. Corollaries~\ref{cor:empirical_ratio_baseline_error} and
\ref{cor:combined_empirical_baseline_error} then transfer the resulting ratio
control to the empirical capability-aware baseline.
\end{remark}

\section{Formulation and Pseudocode of HACPO}
\label{Overall of HACPO}
To facilitate a precise understanding of HACPO, we present the complete algorithmic formulation and training procedure.

Taking two agents (1 and 2) as an example. The optimization objective for agent 1 consists of two terms: the loss computed from its own samples, $J_{\mathrm{homo}}(\theta)$, and the loss computed from samples of other agents, $J_{\mathrm{hete}}(\theta)$. The final loss is the sum of these two terms. Similarly, agent 2 is updated using a loss function of the same form, but with different values.

\begin{equation}
  \mathcal{J}^{(1)}_{homo}=\frac{1}{G}\sum_{i=1}^{G} \min\left(s_{t,i}^{(1,1)} \cdot A_{t,i}^{(1)}(y_{t,i}^{(1)}), \; clip(s_{t,i}^{(1,1)}, \;1-\epsilon_l, \; 1+\epsilon_h) \cdot A_{t,i}^{(1)}\right) 
\end{equation}

$\mathcal{J}^{(1)}_{homo}$ is the homo objective for Agent 1 using its own rollouts.

\begin{equation}
    \mathcal{J}^{(1)}_{hete}=\frac{1}{G}\sum_{i=1}^{G}\left[clip(s_{t,i}^{(1,2)}, \; 1.0-\delta+m\cdot\delta_{step}, \; 1.0)\cdot sg(s_{t,i}^{(1,2)})^{\alpha}\cdot\omega_t^{(2,1)}\cdot A_{t,i}^{(1)}(y_{t,i}^{(2)})\right]
\end{equation}

$\mathcal{J}^{(1)}_{hete}$ is the hete objective for Agent 1 using the rollouts from Agent 2.

\begin{equation}
    A_{t,i}^{(1)}(y_{t,i}^{(1)})=\frac{R(y_{t,i}^{(1)})-\mu_t^{(1)}}{std \{ {\mathcal{R}_{t}(x)} \}},\quad A_{t,i}^{(1)}(y_{t,i}^{(2)})=\frac{R(y_{t,i}^{(2)})-\mu_t^{(1)}}{std \{ {\mathcal{R}_{t}(x)} \}}
\end{equation}

\begin{equation}
    s_{t,i}^{(1,1)}=\left(\frac{\pi_\theta^{(1)}(y_{t,i}^{(1)})}{\pi_{\theta_{old}}^{(1)}(y_{t,i}^{(1)})}\right)^{\frac{1}{|y_{t,i}^{(1)}|}},\quad s_{t,i}^{(1,2)}=\frac{\pi_\theta^{(1)}(y_{t,i}^{(2)})^{\frac{1}{L_1}}}{\pi_{\theta_{old}}^{(2)}(y_{t,i}^{(2)})^{\frac{1}{L_2}}}
\end{equation}
Here, $L_1$ and $L_2$ respectively represent the length of response $y_{t,i}^{(2)}$ under tokenizers of agent 1 and 2.
\begin{equation}
\mathcal{J}^{(1)}=\mathcal{J}^{(1)}_{homo}+\mathcal{J}^{(1)}_{hete}  
\end{equation}

The final optimization objective is the sum of homogeneous and heterogeneous objective.

\begin{algorithm} [!h]
\caption{Heterogeneous Agent Collaborative Policy  Optimization}
\begin{algorithmic}[1]
\REQUIRE n initial policy models $\pi_{\theta_1}, \pi_{\theta_2}, ...\pi_{\theta_n}$, reward models $R$, task prompts $\mathcal{D}$, each prompt has G outputs. The training step is $t$. Each step has $M$ policy updates.
\FOR{k = 1 to n}
    \STATE rollout policy model $\pi^{(k)}_{\theta_{old}} \leftarrow \pi_{\theta_k}$
\ENDFOR
\FOR{t = 1 to $N$}
    \STATE Sample a batch $\mathcal{D}_{t}$ from $\mathcal{D}$
    \FOR{k = 1 to n}
    \STATE Update the rollout policy model $\pi^{(k)}_{\theta_\text{old}} \leftarrow \pi^{(k)}_{\theta}$
    \ENDFOR
    \FOR{k = 1 to n}
    \STATE Sample G output $y \sim \pi^{(k)}_{\theta_\text{old}}(\cdot \mid x)$ for each question $x \in \mathcal{D}_{t}$
    \STATE Compute rewards $R(y_i)$ for each output $y_i$ in the batch
    \STATE Compute accuracy for the sampling model
    \ENDFOR
    \FOR{k = 1 to n}
        \STATE Compute $A_{t,i}(y)$ for the response y in batch (agent k)
        \FOR{mini batch = 1 to $M$}
            \STATE Update the policy model $\pi^{(k)}_{\theta}$ by maximizing the HACPO objective 
        \ENDFOR
    \ENDFOR
\ENDFOR
\ENSURE $\pi^{(k)}_{\theta} | k=1,2,...,n$ 
\end{algorithmic}
\end{algorithm}
\section{Additional Experimental Results}
\label{sec: Additional Experimental Results}
\begin{table}[h!]
\centering
\caption{MATH500 accuracy across five random seeds, evaluated with a maximum response length of 8192 tokens. The final column reports the mean and sample standard deviation ($n-1$ denominator).}
\label{tab: seed experiments}
\begin{tabular}{lcccccc}
\toprule
\textbf{Method} & \textbf{0} & \textbf{1} & \textbf{42} & \textbf{1337} & \textbf{3407} & \textbf{Mean $\pm$ SD} \\
\midrule
\multicolumn{7}{c}{\textbf{Qwen3-1.7B-Base}} \\
\midrule
HACPO & \textbf{0.680} & \textbf{0.682} & \textbf{0.690} & \textbf{0.681} & \textbf{0.670} & $0.6806 \pm 0.0071$ \\
GSPO & 0.647 & 0.640 & 0.648 & 0.646 & 0.643 & $0.6448 \pm 0.0033$ \\
GSPO$\times$2 & 0.648 & 0.651 & 0.664 & 0.652 & 0.646 & $0.6522 \pm 0.0070$ \\
\midrule
\multicolumn{7}{c}{\textbf{Qwen3-4B-Base}} \\
\midrule
HACPO & \textbf{0.788} & \textbf{0.803} & \textbf{0.808} & \textbf{0.797} & \textbf{0.799} & $0.7990 \pm 0.0074$ \\
GSPO & 0.765 & 0.769 & 0.782 & 0.768 & 0.768 & $0.7704 \pm 0.0067$ \\
GSPO$\times$2 & 0.769 & 0.774 & 0.756 & 0.777 & 0.775 & $0.7702 \pm 0.0085$ \\
\bottomrule
\end{tabular}
\end{table}
\begin{table}[h!]
\centering
\caption{Runtime Comparison}
\label{tab:runtime}
\begin{tabular}{lccc}
\toprule
\textbf{Method} & \textbf{Qwen3-1.7B-Base} & \textbf{Qwen3-4B-Base} & \textbf{Total Time} \\
\midrule
GSPO                  & 1h 31m & 2h 38m & 4h 9m  \\
GSPO updates$\times$2 & 2h 6m  & 2h 59m & 5h 5m  \\
GSPO$\times$2         & 2h 43m & 4h 1m  & 6h 44m \\
HACPO & \multicolumn{3}{c}{Overall 5h 31m} \\
\bottomrule
\end{tabular}
\end{table}
\begin{table}[h!]
\centering
\caption{GPU Memory Usage}
\label{tab:gpu_memory}
\begin{tabular}{lcc}
\toprule
\textbf{Method} & \textbf{Qwen3-1.7B-Base} & \textbf{Qwen3-4B-Base} \\
\midrule
GSPO                  & 76.8\% & 82.7\% \\
GSPO updates$\times$2 & 76.6\% & 80.3\% \\
GSPO$\times$2         & 82.2\% & 84.2\% \\
HACPO & \multicolumn{2}{c}{Overall 86.0\%} \\
\bottomrule
\end{tabular}
\end{table}
\subsection{Three More Model Combinations}
\label{sec:Three More Model Combinations}
Here, we present additional experiments in Table \ref{Additional Experimental Results}, including comparisons between Qwen3-4B-Base + Qwen3-8B-Base, Llama3.2-1B-Instruct + Llama3.2-3B-Instruct, and Qwen3-1.7B-Base + Llama3.2-1B-Instruct.

\subsection{The Performance over Different Seeds}
\label{sec: The Performance over Different Seeds}
We evaluate Qwen3-1.7B-Base and Qwen3-4B-Base on MATH500 across five random seeds (0, 1, 42, 1337, and 3407). All checkpoints are evaluated with a maximum response length of 8192 tokens, matching the main results, while the training response-length limit remains 4096. The seed-42 entries match the corresponding MATH500 results in Table~\ref{Main Results}. Table~\ref{tab: seed experiments} reports individual runs and their mean $\pm$ sample standard deviation.

Across all five seeds, HACPO consistently outperforms both GSPO and the resource-equivalent GSPO$\times$2 baseline. On Qwen3-1.7B-Base, HACPO achieves $68.06\% \pm 0.71\%$, compared with $64.48\% \pm 0.33\%$ for GSPO and $65.22\% \pm 0.70\%$ for GSPO$\times$2. On Qwen3-4B-Base, HACPO achieves $79.90\% \pm 0.74\%$, compared with $77.04\% \pm 0.67\%$ and $77.02\% \pm 0.85\%$, respectively. The average gains over GSPO$\times$2 are 2.84 and 2.88 percentage points. Two-sided paired $t$-tests on the five seed-matched differences yield $p<0.001$ for 1.7B and $p=0.009$ for 4B. These tests assess the null hypothesis of zero mean improvement, assuming independent runs and approximately normally distributed paired differences; the significance results apply to this model pair on MATH500.
\subsection{The GPU Peak Memory and Overall Runtime}
\label{sec: The GPU Peak Memory and Overall Runtime}
We report the runtime (Table \ref{tab:runtime}) and peak GPU memory utilization ratio (Table \ref{tab:gpu_memory}) for the Qwen3-1.7B-Base and Qwen3-4B-Base combination. We added a baseline GSPO updates $\times$ 2, which fixed the sampling quantity and doubled the number of updates (i.e., both the sampling cost and the update cost were the same as those of HACPO).

\begin{table}[h!]
\caption{Additional Experimental Results}
\label{Additional Experimental Results}
\vskip 0.15in
\begin{center}
\scriptsize
\begin{tabular}{lccccccccr}
\toprule
Model & MATH-500 & math & gsm8k & aime2025 & AMC23 & minerva & olympiad & AVG\\
\midrule
\multicolumn{9}{c}{Qwen3-4B-Base and Qwen3-8B-Base}  \\
\midrule
4B-Base         & 0.61  & 0.676 & 0.445 & 0.1 & 0.4 & 0.308 & 0.347 & 0.412 \\
4B-Base(GRPO)     & 0.796 & 0.788 & 0.885 & \textbf{0.307} & 0.475 & 0.349 & 0.454 & 0.579 \\
4B-Base(GSPO)   & 0.782 & 0.787 & 0.877 & 0.25  & 0.525 & \textbf{0.368} & 0.46 & 0.578\\
4B-Base(GSPO$\times$2)     & 0.756 & 0.794 & 0.873 & 0.208  & 0.55 & 0.382 & 0.463  & 0.575 \\
4B-Base(Naive)  & 0.734 & 0.712 & 0.895 & 0.143 & 0.55  & 0.342 & 0.354 & 0.526 \\
\rowcolor{gray!20} 4B-Base(HACPO)    & \textbf{0.81}  & \textbf{0.803} & \textbf{0.904} & 0.275 & \textbf{0.6} & 0.364 & \textbf{0.463} & \textbf{0.603}\\
8B-Base         & 0.647 & 0.713 & 0.684 & 0.033 & 0.4 & 0.232 & 0.375 & 0.441 \\
8B-Base(GRPO)   & 0.814 & 0.812 & 0.921 & 0.265 & 0.575 & 0.415 & \textbf{0.479} & 0.612 \\
8B-Base(GSPO)   & 0.794 & 0.804 & 0.923 & 0.225 & 0.6   & \textbf{0.426} & 0.468 & 0.606\\
8B-Base(GSPO$\times$2)   & 0.8 & 0.803 & 0.92 & 0.2 & 0.575 & 0.404 & 0.46 & 0.595\\
8B-Base(Naive)  & 0.79  & 0.783 & 0.921 & 0.252 & 0.5 & 0.408 & 0.429 & 0.583\\
\rowcolor{gray!20} 8B-Base(HACPO)    & \textbf{0.828} & \textbf{0.813} & \textbf{0.933} & \textbf{0.323} & \textbf{0.625} & 0.423 & 0.467 & \textbf{0.63} \\
\midrule
\multicolumn{9}{c}{Llama3.2-1B-Instruct and Llama3.2-3B-Instruct}  \\
\midrule
Llama3.2-1B         & 0.176 & 0.297 & 0.489 & 0   & 0.15  & 0.052 & 0.061 & 0.18\\
Llama3.2-1B(GRPO)   & 0.35 & 0.349 & 0.569 & 0 & 0.125 & 0.008 & 0.097 & 0.214\\
Llama3.2-1B(GSPO)   & \textbf{0.356} & 0.346 & 0.523 & 0.021 & 0.125 & 0.066 & 0.088 & 0.218\\
Llama3.2-1B(GSPO$\times$2)   & 0.352 & 0.349 & \textbf{0.573} & 0.07 & 0.125 & 0.079 & \textbf{0.103} & 0.227\\
Llama3.2-1B(Naive)  & 0.284 & 0.302 & 0.45  & 0.0   & 0.025 & 0.066 & 0.073 & 0.171\\
\rowcolor{gray!20} Llama3.2-1B(HACPO)    & 0.35  & \textbf{0.352} & 0.541 & \textbf{0.022} & \textbf{0.2}   & \textbf{0.081} & 0.085 & \textbf{0.233}\\
Llama3.2-3B         & 0.267 & 0.441 & 0.788 & 0.0   & 0.2   & 0.169 & 0.158 & 0.289\\
Llama3.2-3B(GRPO)   & 0.502 & 0.507 & 0.814 & 0.0 & 0.25 & 0.199 & 0.174 & 0.349\\
Llama3.2-3B(GSPO)   & 0.512 & 0.501 & 0.812 & 0.054 & 0.225 & 0.184 & 0.17 & 0.351\\
Llama3.2-3B (GSPO$\times$2)   & 0.488 & 0.498 & \textbf{0.829} & 0.0 & 0.175 & 0.188 & 0.159 & 0.334 \\
Llama3.2-3B(Naive)  & 0.406 & 0.407 & 0.734 & 0.0   & 0.225 & 0.177 & 0.107 & 0.294\\
\rowcolor{gray!20} Llama3.2-3B(HACPO)    & \textbf{0.522} & \textbf{0.51}  & 0.828 & \textbf{0.067} & \textbf{0.275} & \textbf{0.199} & \textbf{0.188} & \textbf{0.37}\\
\midrule
\multicolumn{9}{c}{Qwen3-1.7B-Base and Llama3.2-1B-Instruct}  \\
\midrule
Qwen3        & 0.5   & 0.483 & 0.616 & 0.033 & 0.3   & 0.206 & 0.229 & 0.338 \\
Qwen3(GRPO)   & \textbf{0.682} & 0.652 & 0.824 & 0.16 & 0.375 & 0.272 & 0.298 & 0.466 \\
Qwen3(GSPO)     & 0.648 & 0.641 & 0.826 & 0.148 & 0.45  & 0.272 & 0.287 & 0.467\\
Qwen3(GSPO$\times$2)   & 0.664 & 0.65 & 0.829 & 0.177 & 0.375  & 0.265 & 0.293 & 0.475 \\
Qwen3(Naive)    & 0.59  & 0.596 & 0.798 & 0.105 & 0.3   & 0.221 & 0.241 & 0.407\\
\rowcolor{gray!20} Qwen3(HACPO)      & 0.676 & \textbf{0.661} & \textbf{0.838} & \textbf{0.22}  & \textbf{0.45}  & \textbf{0.305} & \textbf{0.32} & \textbf{0.496} \\
Llama3.2         & 0.176 & 0.297 & 0.489 & 0.033 & 0.15  & 0.052 & 0.061 & 0.18\\
Llama3.2(GRPO)   & 0.35 & 0.349 & \textbf{0.569} & 0 & 0.125 & 0.008 & 0.097 & 0.214\\
Llama3.2(GSPO)   & 0.356 & 0.346 & 0.523 & 0.021 & 0.125 & 0.066 & 0.088 & 0.218\\
Llama3.2(GSPO$\times$2)   & 0.352 & 0.349 & 0.573 & 0.07 & 0.125 & 0.079 & \textbf{0.103} & 0.227\\
Llama3.2(Naive)  & 0.336 & 0.337 & 0.512 & 0.0   & 0.125 & 0.066 & 0.071 & 0.214\\
\rowcolor{gray!20} Llama3.2(HACPO)    & \textbf{0.356} & \textbf{0.368} & 0.533 & \textbf{0.033} & \textbf{0.15}  & \textbf{0.066} & 0.091 & \textbf{0.228}\\
\bottomrule
\end{tabular}
\end{center}
\end{table}
\subsection{Distillation and Frozen-Model Baselines}
\label{sec:distillation_baselines}

We compare HACPO with one-way learning baselines using Qwen3-1.7B-Base and Qwen3-4B-Base. On-policy distillation (OPD)~\cite{on-policy_distillation} and OPD+GRPO update only the 1.7B model, using the frozen 4B base model as the teacher. HACPO (Frozen Strong) also freezes the 4B model, while updating the 1.7B model with HACPO using their shared rollouts. Full HACPO updates both models. Table~\ref{tab:distillation_baselines} reports the resulting 1.7B model's performance on the seven benchmarks; the general training and evaluation settings are described in Appendix~\ref{sec: Training and Evalution Details}.

\begin{table}[htbp]
\centering
\scriptsize
\setlength{\tabcolsep}{3pt}
\caption{Distillation and frozen-model baselines for Qwen3-1.7B-Base, with Qwen3-4B-Base as the stronger model. AVG is the mean accuracy across seven benchmarks.}
\label{tab:distillation_baselines}
\begin{tabular}{lcccccccc}
\toprule
Method & MATH-500 & MATH & GSM8K & AIME2025 & AMC23 & Minerva & Olympiad & AVG \\
\midrule
OPD & 0.620 & 0.582 & 0.700 & 0.033 & 0.400 & 0.254 & 0.252 & 0.406 \\
OPD+GRPO & 0.610 & 0.584 & 0.704 & 0.067 & 0.375 & 0.250 & 0.263 & 0.408 \\
HACPO (Frozen Strong) & 0.672 & 0.643 & 0.829 & 0.033 & 0.400 & 0.279 & 0.308 & 0.452 \\
\rowcolor{gray!20}
HACPO & 0.690 & 0.674 & 0.822 & 0.225 & 0.450 & 0.279 & 0.314 & \textbf{0.493} \\
\bottomrule
\end{tabular}
\end{table}

HACPO achieves an average accuracy of 49.3\%, compared with 40.6\% for OPD, 40.8\% for OPD+GRPO, and 45.2\% for HACPO (Frozen Strong). In this setting, the improvement over the frozen-strong variant supports the benefit of jointly updating both models and allowing their rollout distributions to evolve during training. These comparisons provide evidence that one-way transfer from a fixed base-model teacher does not fully account for HACPO's gains.

\subsection{Sequential Train-then-Distill}
\label{sec:train_then_distill}

We further compare HACPO with a sequential strategy that first improves the stronger model before using it as a teacher. We train Qwen3-4B-Base with GSPO$\times$2, freeze the resulting checkpoint, and then distill Qwen3-1.7B-Base using OPD+GRPO. The final portfolio consists of this frozen 4B checkpoint and the distilled 1.7B model. Consequently, the 4B results for Train-then-Distill in Table~\ref{tab:train_then_distill} are the GSPO$\times$2 results from Table~\ref{Main Results}.

\begin{table}[htbp]
\centering
\scriptsize
\setlength{\tabcolsep}{3pt}
\caption{HACPO versus sequential Train-then-Distill. Parenthesized improvements are absolute accuracy differences relative to Train-then-Distill.}
\label{tab:train_then_distill}
\begin{tabular}{lcccccccc}
\toprule
Method & MATH-500 & MATH & GSM8K & AIME2025 & AMC23 & Minerva & Olympiad & AVG \\
\midrule
\multicolumn{9}{c}{\textbf{Qwen3-1.7B-Base}} \\
\midrule
Train-then-Distill & 0.674 & 0.663 & 0.814 & 0.150 & 0.350 & 0.283 & 0.284 & 0.460 \\
\rowcolor{gray!20}
HACPO & 0.690 & 0.674 & 0.822 & 0.225 & 0.450 & 0.279 & 0.314 & \textbf{0.493 (+0.033)} \\
\midrule
\multicolumn{9}{c}{\textbf{Qwen3-4B-Base}} \\
\midrule
Train-then-Distill & 0.756 & 0.794 & 0.873 & 0.208 & 0.550 & 0.382 & 0.463 & 0.575 \\
\rowcolor{gray!20}
HACPO & 0.808 & 0.801 & 0.903 & 0.267 & 0.575 & 0.386 & 0.467 & \textbf{0.601 (+0.026)} \\
\bottomrule
\end{tabular}
\end{table}

HACPO improves average accuracy by 3.3 percentage points for the 1.7B model and 2.6 percentage points for the 4B model. Under the same hardware configuration, sequential Train-then-Distill takes 6h~26m in total: 4h~01m to train the 4B model and 2h~25m to distill the 1.7B model. HACPO obtains the complete two-model portfolio in 5h~31m, reducing end-to-end training time by 14.2\% in this comparison.

Our HACPO implementation loads, runs, and offloads the models one at a time during rollout generation, and likewise loads, updates, and offloads them sequentially during optimization. This execution avoids keeping both models resident in GPU memory simultaneously. As reported in Table~\ref{tab:gpu_memory}, HACPO's peak GPU memory utilization is 86.0\%, compared with 84.2\% for training the 4B model alone with GSPO$\times$2. The latter measures the large-model training stage; it does not measure the peak memory of the complete Train-then-Distill pipeline.

\subsection{Complementary Rollouts from the Weaker Model}
\label{sec:weak_to_strong}

We analyze the collected rollouts from the Qwen3-1.7B-Base/Qwen3-4B-Base experiment to examine how a weaker model can contribute useful training data to a stronger partner. For 627 of 7,424 problems (8.4\%), the 1.7B model has a higher sampled success rate than the 4B model. Here, the sampled success rate is the fraction of a model's collected rollouts for a given problem that pass verification. Thus, lower average capability does not imply a lower sampled success rate on every problem.

We also observe problems for which the 4B model generates no correct rollout while the 1.7B model generates at least one. These weak-only successes supply verified solutions absent from the stronger model's own sampled rollouts. Together with the 4B model's performance improvement under HACPO, this analysis provides evidence of complementary information from the weaker model. It does not separately quantify the contribution of alternative reasoning paths or incorrect trajectories to the final improvement.

\clearpage
\section*{NeurIPS Paper Checklist}

\begin{enumerate}

\item {\bf Claims}
    \item[] Question: Do the main claims made in the abstract and introduction accurately reflect the paper's contributions and scope?
    \item[] Answer: \answerYes{} 
    \item[] Justification: The abstract and introduction clearly state the claims about the HACRL problem and the HACPO algorithm, which are well-supported by the theoretical proofs in Appendix \ref{sec:oracle_capability_ratio_concentration}
    and the empirical results in Section \ref{sec: Experiment}. Furthermore, we demonstrate the generalizability of our claims by comprehensively evaluating HACPO across three types of agent heterogeneity (state, size, and model), six two-agent combinations and one three-agent combination, seven diverse mathematical reasoning benchmarks and five seeds for the Qwen3-1.7B-Base/Qwen3-4B-Base combination on MATH500, proving that our method is robust and extends well across various heterogeneous settings.

\item {\bf Limitations}
    \item[] Question: Does the paper discuss the limitations of the work performed by the authors?
    \item[] Answer: \answerYes{} 
    \item[] Justification: We explicitly discuss the limitations in Section \ref{sec: Limitation and Future Work}, noting that our empirical evaluation is currently restricted to two- and three-agent scenarios due to engineering bottlenecks within the verl framework and computational resources, and outline scaling as future work.

\item {\bf Theory assumptions and proofs}
    \item[] Question: For each theoretical result, does the paper provide the full set of assumptions and a complete (and correct) proof?
    \item[] Answer: \answerYes{} 
    \item[] Justification: The proof of oracle baseline unbiasedness (Theorem~\ref{thm:oracle_unbiased_baseline}) is provided in Appendix~\ref{sec:oracle_capability_ratio_concentration}; the required assumptions and finite-batch capability ratio error bounds are given in Appendix~\ref{sec: Theoretical Analysis}.

    \item {\bf Experimental result reproducibility}
    \item[] Question: Does the paper fully disclose all the information needed to reproduce the main experimental results of the paper to the extent that it affects the main claims and/or conclusions of the paper (regardless of whether the code and data are provided or not)?
    \item[] Answer: \answerYes{} 
    \item[] Justification: We describe the HACRL problem and HACPO algorithm in the main text and specify the models, datasets, and hyperparameters in Appendix~\ref{sec: Training and Evalution Details}. Our implementation, based on verl, is publicly available at \url{https://github.com/zzx-peter/HACRL-code}.

\item {\bf Open access to data and code}
    \item[] Question: Does the paper provide open access to the data and code, with sufficient instructions to faithfully reproduce the main experimental results, as described in supplemental material?
    \item[] Answer: \answerYes{} 
    \item[] Justification: The HACPO implementation is publicly available at \url{https://github.com/zzx-peter/HACRL-code}, with installation and launch instructions in the repository README. The models and datasets used in our experiments are publicly available and identified in Appendix~\ref{sec: Training and Evalution Details}.

\item {\bf Experimental setting/details}
    \item[] Question: Does the paper specify all the training and test details (e.g., data splits, hyperparameters, how they were chosen, type of optimizer) necessary to understand the results?
    \item[] Answer: \answerYes{} 
    \item[] Justification: Detailed experimental settings, including maximum prompt/response lengths, batch sizes and clipping parameters for different combinations, are specified in Appendix \ref{sec: Training and Evalution Details}.

\newpage
\item {\bf Experiment statistical significance}
    \item[] Question: Does the paper report error bars suitably and correctly defined or other appropriate information about the statistical significance of the experiments?
    \item[] Answer: \answerYes{} 
    \item[] Justification: To ensure statistical significance and robustness against random initialization, we evaluate the Qwen3-1.7B-Base/Qwen3-4B-Base combination on MATH500 across five different random seeds (0, 1, 42, 1337, and 3407), with the detailed performance distributions reported in Appendix \ref{sec: The Performance over Different Seeds} to account for variance.

\item {\bf Experiments compute resources}
    \item[] Question: For each experiment, does the paper provide sufficient information on the computer resources (type of compute workers, memory, time of execution) needed to reproduce the experiments?
    \item[] Answer: \answerNo{} 
    \item[] Justification: Appendix~\ref{sec: Training and Evalution Details} specifies eight NVIDIA A800 GPUs with 80\,GB each. Appendix~\ref{sec: The GPU Peak Memory and Overall Runtime} reports runtime and peak memory utilization for the Qwen3-1.7B-Base/Qwen3-4B-Base combination; compute profiles for all settings and total project compute are not reported.
    
\item {\bf Code of ethics}
    \item[] Question: Does the research conducted in the paper conform, in every respect, with the NeurIPS Code of Ethics \url{https://neurips.cc/public/EthicsGuidelines}?
    \item[] Answer: \answerYes{} 
    \item[] Justification: This research focuses on foundational algorithmic improvements in reinforcement learning (RLVR) and conforms entirely to the NeurIPS Code of Ethics, with no intent for malicious application.

\item {\bf Broader impacts}
    \item[] Question: Does the paper discuss both potential positive societal impacts and negative societal impacts of the work performed?
    \item[] Answer: \answerNo{} 
    \item[] Justification: As this is a foundational machine learning algorithmic study focused on optimizing reinforcement learning efficiency, it does not have direct, immediate negative societal applications. Therefore, a dedicated broader impacts section was not included.
    
\item {\bf Safeguards}
    \item[] Question: Does the paper describe safeguards that have been put in place for responsible release of data or models that have a high risk for misuse (e.g., pre-trained language models, image generators, or scraped datasets)?
    \item[] Answer: \answerNA{} 
    \item[] Justification: This work focuses on a training paradigm and does not release new pre-trained models, high-risk systems, or scraped datasets.

\item {\bf Licenses for existing assets}
    \item[] Question: Are the creators or original owners of assets (e.g., code, data, models), used in the paper, properly credited and are the license and terms of use explicitly mentioned and properly respected?
    \item[] Answer: \answerNo{} 
    \item[] Justification: We use publicly available datasets, open-weight models, and the verl framework and cite their sources. Asset-specific license names and terms of use are not exhaustively listed in the paper.

\item {\bf New assets}
    \item[] Question: Are new assets introduced in the paper well documented and is the documentation provided alongside the assets?
    \item[] Answer: \answerYes{} 
    \item[] Justification: Our HACPO implementation modifies verl and is publicly released at \url{https://github.com/zzx-peter/HACRL-code}; the paper results correspond to the main branch. The README documents installation, launching, and algorithmic changes, and the repository includes an Apache-2.0 license.

\item {\bf Crowdsourcing and research with human subjects}
    \item[] Question: For crowdsourcing experiments and research with human subjects, does the paper include the full text of instructions given to participants and screenshots, if applicable, as well as details about compensation (if any)? 
    \item[] Answer: \answerNA{} 
    \item[] Justification: The research does not involve crowdsourcing or human subjects.

\item {\bf Institutional review board (IRB) approvals or equivalent for research with human subjects}
    \item[] Question: Does the paper describe potential risks incurred by study participants, whether such risks were disclosed to the subjects, and whether Institutional Review Board (IRB) approvals (or an equivalent approval/review based on the requirements of your country or institution) were obtained?
    \item[] Answer: \answerNA{} 
    \item[] Justification: The research does not involve human subjects.

\item {\bf Declaration of LLM usage}
    \item[] Question: Does the paper describe the usage of LLMs if it is an important, original, or non-standard component of the core methods in this research? Note that if the LLM is used only for writing, editing, or formatting purposes and does \emph{not} impact the core methodology, scientific rigor, or originality of the research, declaration is not required.
    \item[] Answer: \answerYes{} 
    \item[] Justification: Qwen3, Llama3.2, and Gemma3 serve as the agents trained and evaluated in our method; their roles and experimental settings are described in Section~\ref{sec: method} and Appendix~\ref{sec: Training and Evalution Details}. LLM assistance in manuscript preparation was limited to writing, editing, and formatting.

\end{enumerate}

\end{document}